\pdfoutput=1

\documentclass{article}
\usepackage[margin=1in]{geometry}
\usepackage[parfill]{parskip}
\usepackage{times}
\usepackage{graphicx}
\usepackage{caption}
\usepackage{subcaption}
\usepackage{float}
\usepackage{algorithm}
\usepackage{algorithmic}
\usepackage{amsfonts,amsmath,amsthm,amssymb}

\usepackage[shortcuts]{extdash}
\newcommand{\SREC}{S\=/REC}

\newtheoremstyle{exampstyle}
  {10pt} 
  {10pt} 
  {} 
  {} 
  {\bfseries} 
  {.} 
  {.5em} 
  {} 
\theoremstyle{exampstyle} \newtheorem{theorem}{Theorem}[section]

\newtheorem{lemma}[theorem]{Lemma}
\newtheorem{definition}{Definition}

\DeclareMathOperator*{\argmin}{arg\,min}
\title{Compressed Sensing using Generative Models}
\date{}

\author{
Ashish Bora\thanks{University of Texas at Austin, Department of Computer Science, email: $\mathsf{ashish.bora@utexas.edu}$} \and
Ajil Jalal\thanks{University of Texas at Austin, Department of Electrical and Computer Engineering, email: $\mathsf{ajiljalal@utexas.edu}$} \and
Eric Price\thanks{University of Texas at Austin, Department of Computer Science, email: $\mathsf{ecprice@cs.utexas.edu}$} \and
Alexandros G. Dimakis\thanks{University of Texas at Austin, Department of Electrical and Computer Engineering, email: $\mathsf{dimakis@austin.utexas.edu}$}
}

\newcommand{\norm}[1]{\|#1\|}
\newcommand{\R}{\mathbb{R}}
\newcommand{\eps}{\epsilon}
\newcommand{\wh}[1]{\widehat{#1}}

\usepackage{hyperref}

\begin{document}
\maketitle

\begin{abstract}
    The goal of compressed sensing is to estimate a vector from an underdetermined system of noisy linear measurements, by making use of prior knowledge on the structure of vectors in the relevant domain. For almost all results in this literature, the structure is represented by sparsity in a well-chosen basis. We show how to achieve guarantees similar to standard compressed sensing but without employing sparsity at all.  Instead, we suppose that vectors lie near the range of a generative model $G: \R^k \to \R^n$.  Our main theorem is that, if $G$ is $L$-Lipschitz, then roughly $O(k \log L)$ random Gaussian measurements suffice for an $\ell_2/\ell_2$ recovery guarantee. We demonstrate our results using generative models from published variational autoencoder and generative adversarial networks. Our method can use $5$-$10$x fewer measurements than Lasso for the same accuracy.
\end{abstract}

\section{Introduction}

Compressive or compressed sensing is the problem of reconstructing an unknown vector $x^* \in \mathbb{R}^n$ after observing $m<n $ linear measurements of its entries, possibly with added noise:
\[
    y = Ax^* + \eta,
\]
where $A \in \mathbb{R}^{m \times n}$ is called the measurement matrix and $\eta \in \mathbb{R}^m$ is noise.  Even without noise, this is an underdetermined system of linear equations, so recovery is impossible unless we make an assumption on the structure of the unknown vector $x^*$.  We need to assume that the unknown vector is ``natural,'' or ``simple,'' in some application-dependent way.

The most common structural assumption is that the vector $x^*$ is $k$-sparse in some known basis (or approximately $k$-sparse). Finding the sparsest solution to an underdetermined system of linear equations is NP-hard, but still convex optimization can provably recover the true sparse vector $x^*$ if the matrix $A$ satisfies conditions such as the Restricted Isometry Property (RIP) or the related Restricted Eigenvalue Condition (REC)~\cite{tibshirani1996regression,candes2006stable,donoho2006compressed,bickel2009simultaneous}. The problem is also called high-dimensional sparse linear regression and there is vast literature on establishing conditions for different recovery algorithms, different assumptions on the design of $A$ and generalizations of RIP and REC for other structures, see \textit{e.g.}~\cite{bickel2009simultaneous,negahban2009unified,agarwal2010fast,loh2011high,
bach2012optimization}.

This significant interest is justified since a large number of applications can be expressed as recovering an unknown vector from noisy linear measurements. For example, many tomography problems can be expressed in this framework: $x^*$ is the unknown true tomographic image and the linear measurements are obtained by x-ray or other physical sensing system that produces sums or more general linear projections of the unknown pixels. Compressed sensing has been studied extensively for medical applications including computed tomography (CT)~\cite{chen2008prior}, rapid MRI~\cite{lustig2007sparse} and neuronal spike train recovery~\cite{hegde2009compressive}. Another impressive application is the ``single pixel camera''~\cite{duarte2008single}, where digital micro-mirrors provide linear combinations to a single pixel sensor that then uses compressed sensing reconstruction algorithms to reconstruct an image. These results have been extended by combining sparsity with additional structural assumptions~\cite{baraniuk2010model,hegde2015nearly}, and by generalizations such as translating sparse vectors into low-rank matrices~\cite{negahban2009unified,bach2012optimization,foygel2014corrupted}. These results can improve performance when the structural assumptions fit the sensed signals.  Other works perform ``dictionary learning,'' seeking overcomplete bases where the data is more sparse (see~\cite{chen2015compressed} and references therein).

In this paper instead of relying on sparsity, we use structure from a \textit{generative model}.  Recently, several neural network based generative models such as variational auto-encoders (VAEs)~\cite{kingma2013auto} and generative adversarial networks (GANs)~\cite{goodfellow2014generative} have found success at modeling data distributions. In these models, the generative part learns a
mapping from a low dimensional representation space $z \in \mathbb{R}^k$ to the high dimensional sample space $G(z) \in \mathbb{R}^n$. While training, this mapping is encouraged to produce vectors that resemble the vectors in the training dataset. We can therefore use any pre-trained generator to approximately capture the notion of an vector being ``natural'' in our domain: the generator defines a probability distribution over vectors in sample space and tries to assign higher probability to more likely vectors, for the dataset it has been trained on.  We expect that vectors ``natural'' to our domain will be close to some point in the support of this distribution, \textit{i.e.}, in the range of $G$.

\noindent \textbf{Our Contributions:} We present an algorithm that uses generative models for compressed sensing.  Our algorithm simply uses gradient descent to optimize the representation $z \in \R^k$ such that the corresponding image $G(z)$ has small measurement error $\norm{AG(z) - y}_2^2$.  While this is a nonconvex objective to optimize, we empirically find that gradient descent works well, and the results can significantly outperform Lasso with relatively few measurements.

We obtain theoretical results showing that, as long as gradient descent finds a good approximate solution to our objective, our output $G(z)$ will be almost as close to the true $x^*$ as the closest possible point in the range of $G$.

The proof is based on a generalization of the Restricted Eigenvalue Condition ($REC$) that we call the Set-Restricted Eigenvalue Condition (\SREC{}).  Our main theorem is that if a measurement matrix satisfies the \SREC{} for the range of a given generator $G$, then the measurement error minimization optimum is close to the true $x^*$. Furthermore, we show that random Gaussian measurement matrices satisfy the \SREC{} condition with high probability for large classes of generators.  Specifically, for $d$-layer neural networks such as VAEs and GANs, we show that $O(kd \log n)$ Gaussian measurements suffice to guarantee good reconstruction with high probability. One result, for ReLU-based networks, is the following:

\begin{theorem}\label{thm:intro}
    Let $G: \R^k \to \R^n$ be a generative model from a $d$-layer neural network using ReLU activations.  Let $A \in \R^{m\times n}$ be a random Gaussian matrix for $m = O(k d \log n)$, scaled so $A_{i,j} \sim N(0, 1/m)$. For any $x^* \in \R^n$ and any observation $y = Ax^* + \eta$, let $\wh{z}$ minimize $\norm{y - AG(z)}_2$ to within additive $\eps$ of the optimum.  Then with $1 - e^{-\Omega(m)}$ probability,
    \[
        \norm{G(\wh{z})-x^*}_2 \leq 6\min_{z^* \in \R^k} \norm{G(z^*) - x^*}_2 + 3\norm{\eta}_2 + 2\eps.
    \]
\end{theorem}
Let us examine the terms in our error bound in more detail. The first two are the minimum possible error of any vector in the range of the generator and the norm of the noise; these are necessary for such a technique, and have direct analogs in standard compressed sensing guarantees. The third term $\eps$ comes from gradient descent not necessarily converging to the global optimum; empirically, $\eps$ does seem to converge to zero, and one can check post-observation that this is small by computing the upper bound $\norm{y - AG(\wh{z})}_2$.

While the above is restricted to ReLU-based neural networks, we also show similar results for arbitrary $L$-Lipschitz generative models, for $m \approx O(k \log L)$.  Typical neural networks have $\text{poly}(n)$-bounded weights in each layer, so $L \leq n^{O(d)}$, giving for all activation functions the same $O(kd \log n)$ sample complexity as for ReLU networks.

\begin{theorem}\label{thm:intro2}
    Let $G: \R^k \to \R^n$ be an $L$-Lipschitz function. Let $A \in \R^{m\times n}$ be a random Gaussian matrix for $m = O(k \log \frac{Lr}{\delta})$, scaled so $A_{i,j} \sim N(0, 1/m)$. For any $x^* \in \R^n$ and any observation $y = Ax^* + \eta$, let $\wh{z}$ minimize $\norm{y - AG(z)}_2$ to within additive $\eps$ of the optimum over vectors with $\norm{\wh{z}}_2 \leq r$.  Then with $1 - e^{-\Omega(m)}$ probability,
    \[
        \norm{G(\wh{z})-x^*}_2 \leq 6\min_{\substack{z^* \in \R^k\\\norm{z^*}_2 \leq r}} \norm{G(z^*) - x^*}_2 + 3\norm{\eta}_2 + 2\eps + 2\delta.
    \]
\end{theorem}
The downside is two minor technical conditions: we only optimize over representations $z$ with $\norm{z}$ bounded by $r$, and our error gains an additive $\delta$ term.  Since the dependence on these parameters is $\log (rL/\delta)$, and $L$ is something like $n^{O(d)}$, we may set $r = n^{O(d)}$ and $\delta = 1/n^{O(d)}$ while only losing constant factors, making these conditions very mild. In fact, generative models normally have the coordinates of $z$ be independent uniform or Gaussian, so $\norm{z} \approx \sqrt{k} \ll n^d$, and a constant signal-to-noise ratio would have $\norm{\eta}_2 \approx \norm{x^*} \approx \sqrt{n} \gg 1/n^d$.

We remark that, while these theorems are stated in terms of Gaussian matrices, the proofs only involve the distributional Johnson-Lindenstrauss property of such matrices. Hence the same results hold for matrices with subgaussian entries or fast-JL matrices~\cite{ailon2009fast}.

\section{Our Algorithm}

All norms are $2$-norms unless specified otherwise.

Let $x^* \in \mathbb{R}^n$ be the vector we wish to sense. Let $A \in \mathbb{R}^{m \times n}$ be the measurement matrix and $\eta \in \mathbb{R}^m$ be the noise vector. We observe the measurements $y = Ax^* + \eta$. Given $y$ and $A$, our task is to find a reconstruction $\hat{x}$  close to $x^*$.

A generative model is given by a deterministic function $G: \mathbb{R}^k \rightarrow \mathbb{R}^n$, and a distribution $P_Z$ over $z \in \mathbb{R}^k$. To generate a sample from the generator, we can draw $z \sim P_Z$ and the sample then is $G(z)$. Typically, we have $k \ll n$, \textit{i.e.} the generative model maps from a low dimensional representation space to a high dimensional sample space.

Our approach is to find a vector in representation space such that the corresponding vector in the sample space matches the observed measurements. We thus define the objective to be
\begin{equation} \label{eqn:loss}
    loss(z) = \|AG(z) - y\|^2
\end{equation}
By using any optimization procedure, we can minimize $loss(z)$ with respect to $z$. In particular, if the generative model $G$ is differentiable, we can evaluate the gradients of the loss with respect to $z$ using backpropagation and use standard gradient based optimizers. If the optimization procedure terminates at $\hat{z}$, our reconstruction for $x^*$ is $G(\hat{z})$. We define the measurement error to be $\|AG(\hat{z}) - y\|^2$ and the reconstruction error to be $\|G(\hat{z}) - x^*\|^2$.

\section{Related Work}

Several recent lines of work explore generative models for reconstruction. The first line of work attempts to project an image on to the representation space of the generator. These works assume full knowledge of the image, and are special cases of the linear measurements framework where the measurement matrix $A$ is identity. Excellent reconstruction results with SGD in the representation space to find an image in the generator range have been reported by~\cite{lipton2017precise} with stochastic clipping and~\cite{creswell2016inverting} with logistic measurement loss. A different approach is introduced in~\cite{dumoulin2016adversarially} and~\cite{donahue2016adversarial}. In their method, a recognition network that maps from the sample space vector $x$ to the representation space vector $z$ is learned jointly with the generator in an adversarial setting.

A second line of work explores reconstruction with structured partial observations.  The inpainting problem consists of predicting the values of missing pixels given a part of the image. This is a special case of linear measurements where each measurement corresponds to an observed pixel. The use of Generative models for this task has been studied in~\cite{yeh2016semantic}, where the objective is taken to be a combination of $L_1$ error in measurements and a perceptual loss term given by the discriminator. Super-resolution is a related task that attempts to increase the resolution of an image. We can view this problem as observing local spatial averages of the unknown higher resolution image and hence cast this as another special case of linear measurements. For prior work on super-resolution see \textit{e.g.}~\cite{yang2010image,dong2016image,kim2016accurate} and references therein.

We also take note of the related work of~\cite{gilbert2017} that connects model-based compressed sensing with the invertibility of Convolutional Neural Networks.

A related result appears in~\cite{baraniuk2009random}, which studies the measurement complexity of an RIP condition for smooth manifolds. This is analogous to our \SREC{} for the range of $G$, but the range of $G$ is neither smooth (because of ReLUs) nor a manifold (because of self-intersection).  Their recovery result was extended in~\cite{hegde2012signal} to unions of two manifolds.

\section{Theoretical Results}

We begin with a brief review of the Restricted Eigenvalue Condition (REC) in standard compressed sensing. The REC is a sufficient condition on $A$ for robust recovery to be possible. The REC essentially requires that all ``approximately sparse'' vectors are far from the nullspace of the matrix $A$. More specifically, $A$ satisfies REC for a constant $\gamma > 0$ if for all approximately sparse vectors $x$,
\begin{equation}\label{eqn:REC}
    \|A x\| \geq \gamma \|x \|.
\end{equation}
It can be shown that this condition is sufficient for recovery of sparse vectors using Lasso. If one examines the structure of Lasso recovery proofs, a key property that is used is that the difference of any two sparse vectors is also approximately sparse (for sparsity up to $2k$). This is a coincidence that is particular to sparsity. By contrast, the difference of two vectors ``natural'' to our domain may not itself be natural. The condition we need is that the difference of any two natural vectors is far from the nullspace of $A$.

We propose a generalized version of the REC for a set $S \subseteq \R^n$ of vectors, the Set-Restricted Eigenvalue Condition (\SREC{}):
\begin{definition}\label{def:rec_def}
  Let $S \subseteq \mathbb{R}^n$.  For some parameters $\gamma > 0$,
  $\delta \geq 0$, a matrix $A \in \mathbb{R}^{m \times n}$ is said to
  satisfy the $\SREC{}(S, \gamma, \delta)$ if  $\ \forall \ x_1, x_2 \in S$,
  $$\|A (x_1 - x_2)\| \geq \gamma \|x_1 - x_2 \| - \delta.$$
\end{definition}
There are two main differences between the \SREC{} and the standard REC in compressed sensing.  First, the condition applies to differences of vectors in an \emph{arbitrary} set $S$ of ``natural'' vectors, rather than just the set of approximately $k$-sparse vectors in some basis.  This will let us apply the definition to $S$ being the range of a generative model.

Second, we allow an additive slack term $\delta$. This is necessary for us to achieve the \SREC{} when $S$ is the output of general Lipschitz functions. Without it, the \SREC{} depends on the behavior of $S$ at arbitrarily small scales. Since there are arbitrarily many such local regions, one cannot guarantee the existence of an $A$ that works for all these local regions.  Fortunately, as we shall see, poor behavior at a small scale $\delta$ will only increase our error by $\mathcal{O}(\delta)$.

The \SREC{} definition requires that for any two vectors in $S$, if they are significantly different (so the right hand side is large), then the corresponding measurements should also be significantly different (left hand side). Hence we can hope to approximate the unknown vector from the measurements, if the measurement matrix satisfies the \SREC{}.

But how can we find such a matrix? To answer this, we present two lemmas showing that random Gaussian matrices of relatively few measurements $m$ satisfy the \SREC{} for the outputs of large and practically useful classes of generative models $G: \R^k \to \R^n$.

In the first lemma, we assume that the generative model $G(\cdot)$ is $L$-Lipschitz, \textit{i.e.}, $\forall \ z_1, z_2 \in \mathbb{R}^k$, we have
\[
    \| G(z_1) - G(z_2)\| \leq L \|z_1 - z_2\|.
\]
Note that state of the art neural network architectures with linear layers, (transposed) convolutions, max-pooling, residual connections, and all popular non-linearities satisfy this assumption. In Lemma~\ref{lemma:lipschitz} in the Appendix we give a simple bound on $L$ in
terms of parameters of the network; for typical networks this is $n^{O(d)}$. We also require the input $z$ to the generator to have bounded norm. Since generative models such as VAEs and GANs typically assume their input $z$ is drawn with independent uniform or Gaussian inputs, this only prunes an exponentially unlikely fraction of the possible outputs.

\begin{lemma}\label{thm:rec_with_p}
    Let $G:\mathbb{R}^k \rightarrow \mathbb{R}^n$ be $L$-Lipschitz. Let $$B^k(r) = \lbrace z \ \vert \ z \in \mathbb{R}^k, \| z \| \leq r \rbrace$$ be an $L_2$-norm ball in $\mathbb{R}^k$.
    For $\alpha < 1$, if
    \[
        m = \Omega \left(\frac{k}{\alpha^2} \log \frac{Lr}{\delta} \right),
    \]
    then a random matrix $A \in \mathbb{R}^{m\times n}$ with IID entries such that $A_{ij} \sim \mathcal{N}\left(0,\frac{1}{m}\right)$ satisfies the $\SREC{}(G(B^k(r)), 1-\alpha, \delta)$ with $1 - e^{-\Omega(\alpha^2 m)}$ probability.
\end{lemma}
All proofs, including this one, are deferred to Appendix A.

Note that even though we proved the lemma for an $L_2$ ball, the same technique works for any compact set.

For our second lemma, we assume that the generative model is a neural network with such that each layer is a composition of a linear transformation followed by a pointwise non-linearity. Many common generative models have such architectures. We also assume that all non-linearities are piecewise linear with at most two pieces. The popular ReLU or LeakyReLU non-linearities satisfy this assumption. We do not make any other assumption, and in particular, the magnitude of the weights in the network do not affect our guarantee.

\begin{lemma}\label{thm: subspace embedding}
    Let $G:\mathbb{R}^k \rightarrow \mathbb{R}^n$ be a $d$-layer neural network, where each layer is a linear transformation followed by a pointwise non-linearity. Suppose there are at most $c$ nodes per layer, and the non-linearities are piecewise linear with at most two pieces, and let
    \[
        m = \Omega \left( \frac{1}{\alpha^2} kd\log c \right)\
    \]
    for some $\alpha < 1$.  Then a random matrix $A\in\mathbb{R}^{m\times n}$ with IID entries $A_{ij}\sim\mathcal{N}(0,\frac{1}{m})$ satisfies the $\SREC{}(G(\mathbb{R}^k), 1 - \alpha, 0)$ with $1 - e^{-\Omega(\alpha^2 m)}$ probability.
\end{lemma}

To show Theorems~\ref{thm:intro} and~\ref{thm:intro2}, we just need to show that the \SREC{} implies good recovery. In order to make our error guarantee relative to $\ell_2$ error in the image space $\R^n$, rather than in the measurement space $\R^m$, we also need that $A$ preserves
norms with high probability~\cite{CDD09}. Fortunately, Gaussian matrices (or other distributional JL matrices) satisfy this property.

\begin{lemma}\label{thm:rec_application}
    Let $A \in \mathbb{R}^{m \times n}$ by drawn from a distribution that (1) satisfies the $\SREC{}(S, \gamma, \delta)$ with probability $1-p$ and (2) has for every fixed $x \in \R^n$, $\norm{Ax} \leq 2\norm{x}$ with probability $1-p$.

    For any $x^* \in \R^n$ and noise $\eta$, let $y = Ax^* + \eta$.  Let $\wh{x}$ approximately minimize $\norm{y - Ax}$ over $x \in S$, \textit{i.e.},
    \[
        \norm{y - A\wh{x}} \leq \min_{x \in S}\norm{y - Ax} + \eps.
    \]
    Then,
    \[
        \norm{\wh{x} - x^*} \leq \left( \frac{4}{\gamma} + 1 \right) \min_{x \in S} \norm{x^* - x} + \frac{1}{\gamma}\left(2\norm{\eta} + \eps + \delta\right)
    \]
    with probability $1-2p$.
\end{lemma}

Combining Lemma~\ref{thm:rec_with_p}, Lemma~\ref{thm: subspace embedding}, and Lemma~\ref{thm:rec_application} gives Theorems~\ref{thm:intro} and~\ref{thm:intro2}. In our setting, $S$ is the range of the generator, and $\wh x$ in the theorem above is the reconstruction $G(\wh z)$ returned by our algorithm.

\newsavebox\myboxone
\savebox{\myboxone}{\includegraphics[width=0.48\textwidth]{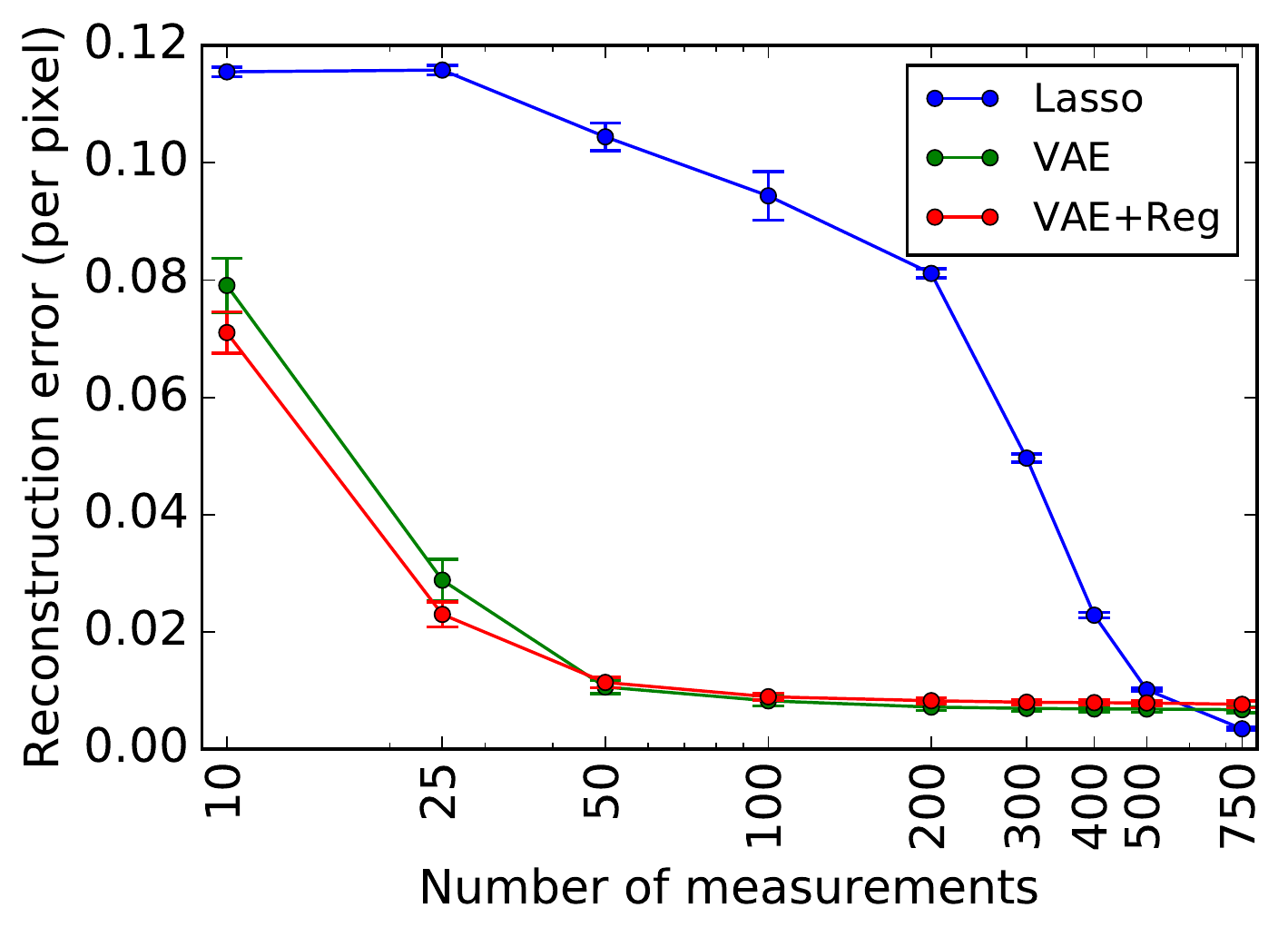}}
\begin{figure*}
    \begin{subfigure}[t]{0.48\textwidth}
        \usebox{\myboxone}
        \vspace*{-3mm}
        \caption{Results on MNIST}
    \label{fig:mnist-reconstr-l2}
    \end{subfigure}\hfill%
    \begin{subfigure}[t]{0.48\textwidth}
        \vbox to \ht\myboxone{%
            \vfill
                    \includegraphics[width=\textwidth]{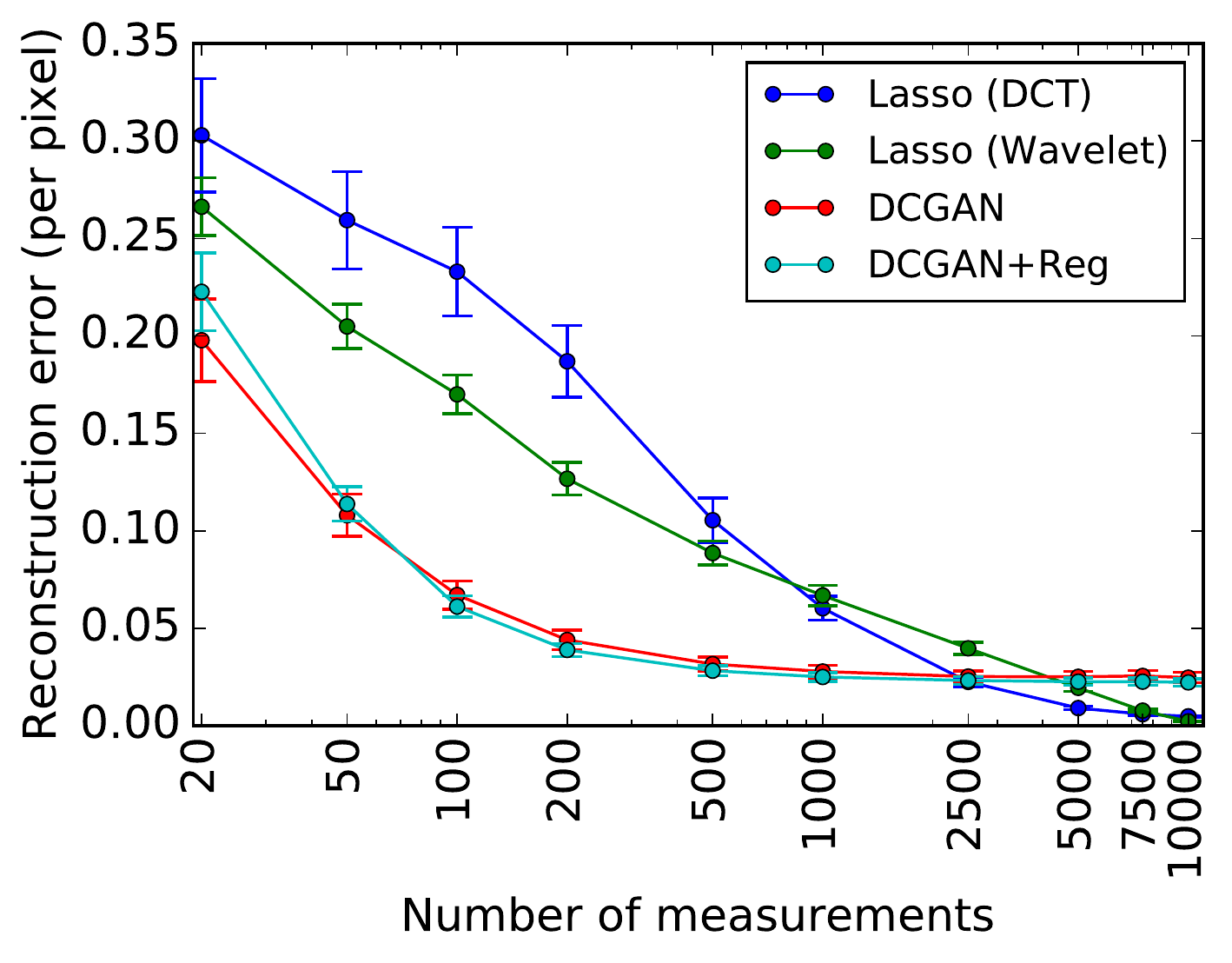}
        }
        \vspace*{2mm}
        \caption{Results on celebA}
    \label{fig:celebA-reconstr-l2}
    \end{subfigure}
    \caption{We compare the performance of our algorithm with baselines. We show a plot of per pixel reconstruction error as we vary the number of measurements. The vertical bars indicate 95\% confidence intervals.}
\label{fig:reconstr-l2}
\end{figure*}

\section{Models}

In this section we describe the generative models used in our experiments. We used two image datasets and two different generative model types (a VAE and a GAN). This provides some evidence that our approach can work with many types of models and datasets.

In our experiments, we found that it was helpful to add a regularization term $L(z)$ to the objective to encourage the optimization to explore more in the regions that are preferred by the respective generative models (see comparison to unregularized versions in Fig.~\ref{fig:reconstr-l2}). Thus the objective function we use for minimization is $$\|AG(z) - y\|^2 + L(z).$$ Both VAE and GAN typically imposes an isotropic Gaussian prior on $z$. Thus $\|z\|^2$ is proportional to the negative log-likelihood under this prior. Accordingly, we use the following regularizer:
\begin{equation}\label{eqn:Lz}
    L(z) = \lambda \|z\|^2,
\end{equation}
where $\lambda$ measures the relative importance of the prior as compared to the measurement error.

\subsection{MNIST with VAE}

The MNIST dataset consists of about $60,000$ images of handwritten digits, where each image is of size $28 \times 28$~\cite{lecun1998gradient}. Each pixel value is either $0$ (background) or $1$ (foreground). No pre-processing was performed. We trained VAE on this dataset. The input to the VAE is a vectorized binary image of input dimension $784$. We set the size of the representation space $k = 20$. The recognition network is a fully connected $784-500-500-20$ network. The generator is also fully connected with the architecture $20-500-500-784$. We train the VAE using the Adam optimizer~\cite{kingma2014adam} with a mini-batch size $100$ and a learning rate of $0.001$.

We found that using $\lambda = 0.1$ in Eqn. (\ref{eqn:Lz}) gave the best performance, and we use this value in our experiments.

The digit images are reasonably sparse in the pixel space. Thus, as a baseline, we use the pixel values directly for sparse recovery using Lasso. We set shrinkage parameter to be $0.1$ for all the experiments.

\subsection{CelebA with DCGAN}

CelebA is a dataset of more than $200,000$ face images of celebrities~\cite{liu2015deep}. The input images were cropped to a $64 \times 64$ RGB image, giving $64 \times 64 \times 3=12288$ inputs per image. Each pixel value was scaled so that all values are between $[-1, 1]$. We trained a DCGAN~\footnote{Code reused from \url{ https://github.com/carpedm20/DCGAN-tensorflow}}~\cite{radford2015unsupervised,carpedm20} on this dataset. We set the input dimension $k=100$ and use a standard normal distribution. The architecture follows that of~\cite{radford2015unsupervised}. The model was trained by one update to the discriminator and two updates to the generator per cycle. Each update used the Adam optimizer~\cite{kingma2014adam} with minibatch size $64$, learning rate $0.0002$ and $\beta_1=0.5$.

We found that using $\lambda = 0.001$ in Eqn.~(\ref{eqn:Lz}) gave the best results and thus, we use this value in our experiments.

For baselines, we perform sparse recovery using Lasso on the images in two domains: (a) 2D Discrete Cosine Transform (2D-DCT) and (b) 2D Daubechies-1 Wavelet Transform (2D-DB1). While the we provide Gaussian measurements of the original pixel values, the $L_1$ penalty is on either the DCT coefficients or the DB1 coefficients of each color channel of an image. For all experiments, we set the shrinkage parameter to be $0.1$ and $0.00001$ respectively for 2D-DCT, and 2D-DB1.

\section{Experiments and Results}

\subsection{Reconstruction from Gaussian measurements}
We take $A$ to be a random matrix with IID Gaussian entries with zero mean and standard deviation of $1/m$. Each entry of noise vector $\eta$ is also an IID Gaussian random variable. We compare performance of different sensing algorithms qualitatively and quantitatively. For quantitative comparison, we use the reconstruction error = $\| \hat{x} - x^* \|^2$, where $\hat{x}$ is an estimate of $x^*$ returned by the algorithm. In all cases, we report the results on a held out test set,  unseen by the generative model at training time.

\subsubsection{MNIST}
The standard deviation of the noise vector is set such that $\sqrt{\mathbb{E}[\| \eta \|^2]} = 0.1$. We use Adam optimizer~\cite{kingma2014adam}, with a learning rate of $0.01$. We do $10$ random restarts with $1000$ steps per restart and pick the reconstruction with best measurement error.

In Fig.~\ref{fig:mnist-reconstr-l2}, we show the reconstruction error as we change the number of measurements both for Lasso and our algorithm. We observe that our algorithm is able to get low errors with far fewer measurements. For example, our algorithm's performance with $25$ measurements matches Lasso's performance with $400$ measurements. Fig.~\ref{fig:mnist-reconstr} shows sample reconstructions by Lasso and our algorithm.

However, our algorithm is limited since its output is constrained to be in the range of the generator.  After $100$ measurements, our algorithm's performance saturates, and additional measurements give no additional performance.  Since Lasso has no such limitation, it eventually surpasses our algorithm, but this takes more than $500$ measurements of the 784-dimensional vector.  We expect that a more powerful generative model with representation dimension $k > 20$ can make better use of additional measurements.

\subsubsection{celebA}

The standard deviation of entries in the noise vector is set such that $\sqrt{\mathbb{E}[\| \eta \|^2]} = 0.01$. We optimize use Adam optimizer~\cite{kingma2014adam}, with a learning rate of $0.1$. We do $2$ random restarts with $500$ update steps per restart and pick the reconstruction with best measurement error.

In Fig.~\ref{fig:celebA-reconstr-l2}, we show the reconstruction error as we change the number of measurements both for Lasso and our algorithm. In Fig.~\ref{fig:celebA-reconstr} we show sample reconstructions by Lasso and our algorithm. We observe that our algorithm is able to produce reasonable reconstructions with as few as $500$ measurements, while the output of the baseline algorithms is quite blurry. Similar to the results on MNIST, if we continue to give more measurements, our algorithm saturates, and for more than $5000$ measurements, Lasso gets a better reconstruction. We again expect that a more powerful generative model with $k > 100$ would perform better in the high-measurement regime.

\begin{figure*}
    \begin{subfigure}[t]{0.48\textwidth}
        \includegraphics[width=\textwidth]{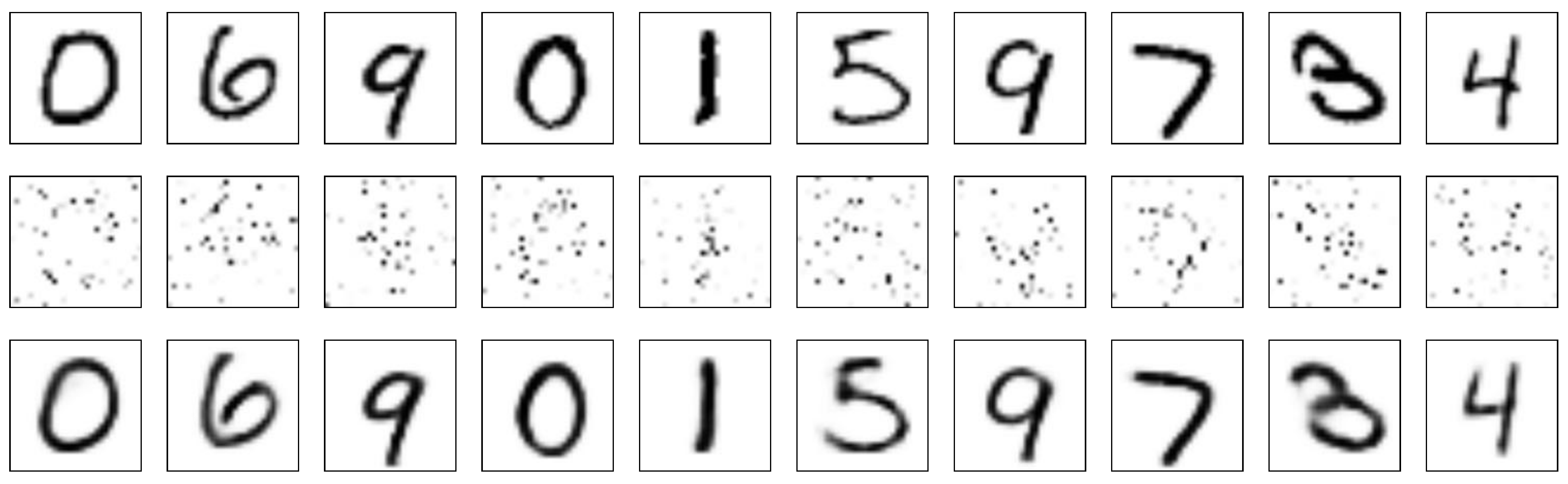}
        \caption{We show original images (top row) and reconstructions by Lasso (middle row) and our algorithm (bottom row).}
        \label{fig:mnist-reconstr}
    \end{subfigure}\hfill%
    \begin{subfigure}[t]{0.48\textwidth}
        \includegraphics[width=\textwidth]{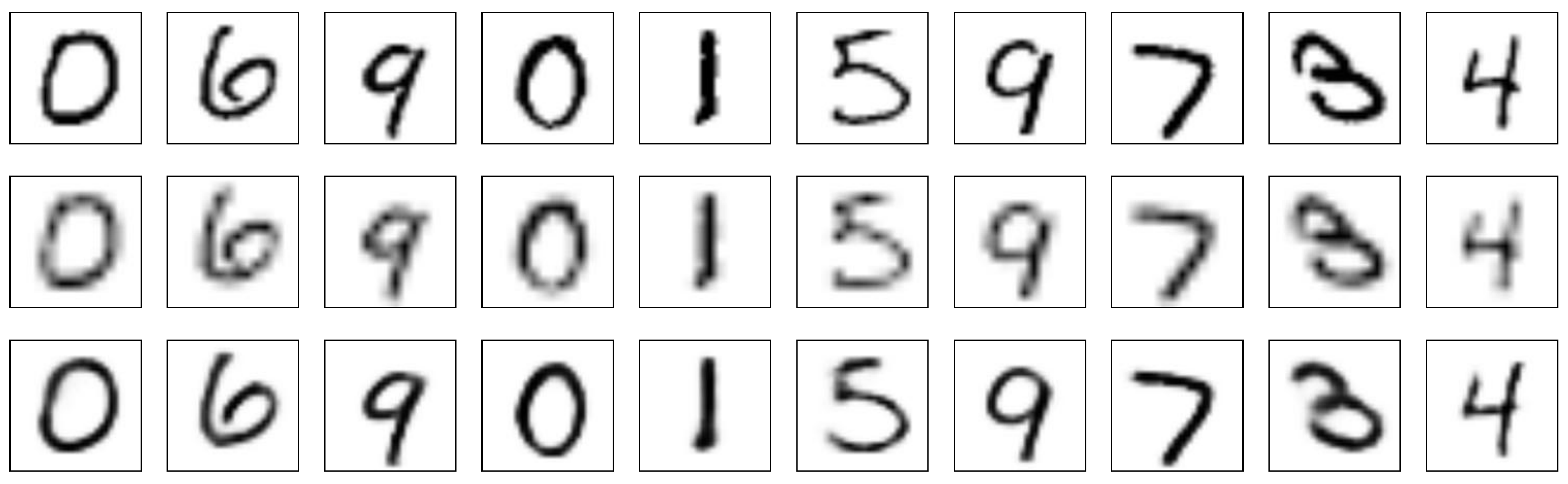}
        \caption{We show original images (top row), low resolution version of original images (middle row) and reconstructions (last row).}
        \label{fig:mnist-superres}
    \end{subfigure}
    \caption{Results on MNIST\@. Reconstruction with 100 measurements (left) and Super-resolution (right)}
\end{figure*}

\subsection{Super-resolution}

Super-resolution is the task of constructing a high resolution image from a low resolution version of the same image. This problem can be thought of as special case of our general framework of linear measurements, where the measurements correspond to local spatial averages of the pixel values. Thus, we try to use our recovery algorithm to perform this task with measurement matrix $A$ tailored to give only the relevant observations. We note that this measurement matrix may not satisfy the \SREC{} condition (with good constants $\gamma$ and $\delta$), and consequently, our theorems may not be applicable.

\subsubsection{MNIST}
We construct a low resolution image by spatial $2 \times 2$ pooling with a stride of $2$ to produce a $14 \times 14$ image. These measurements are used to reconstruct the original $28 \times 28$ image. Fig.~\ref{fig:mnist-superres} shows reconstructions produced by our algorithm on images from a held out test set. We observe sharp reconstructions which closely match the fine structure in the ground truth.

\subsubsection{celebA}

We construct a low resolution image by spatial $4 \times 4$ pooling with a stride of $4$ to produce a $16 \times 16$ image. These measurements are used to reconstruct the original $64 \times 64$ image. In Fig.~\ref{fig:celebA-superres} we show results on images from a held out test set. We see that our algorithm is able to fill in the details to match the original image.

\begin{figure*}
    \centering
    \includegraphics[width=0.9\textwidth]{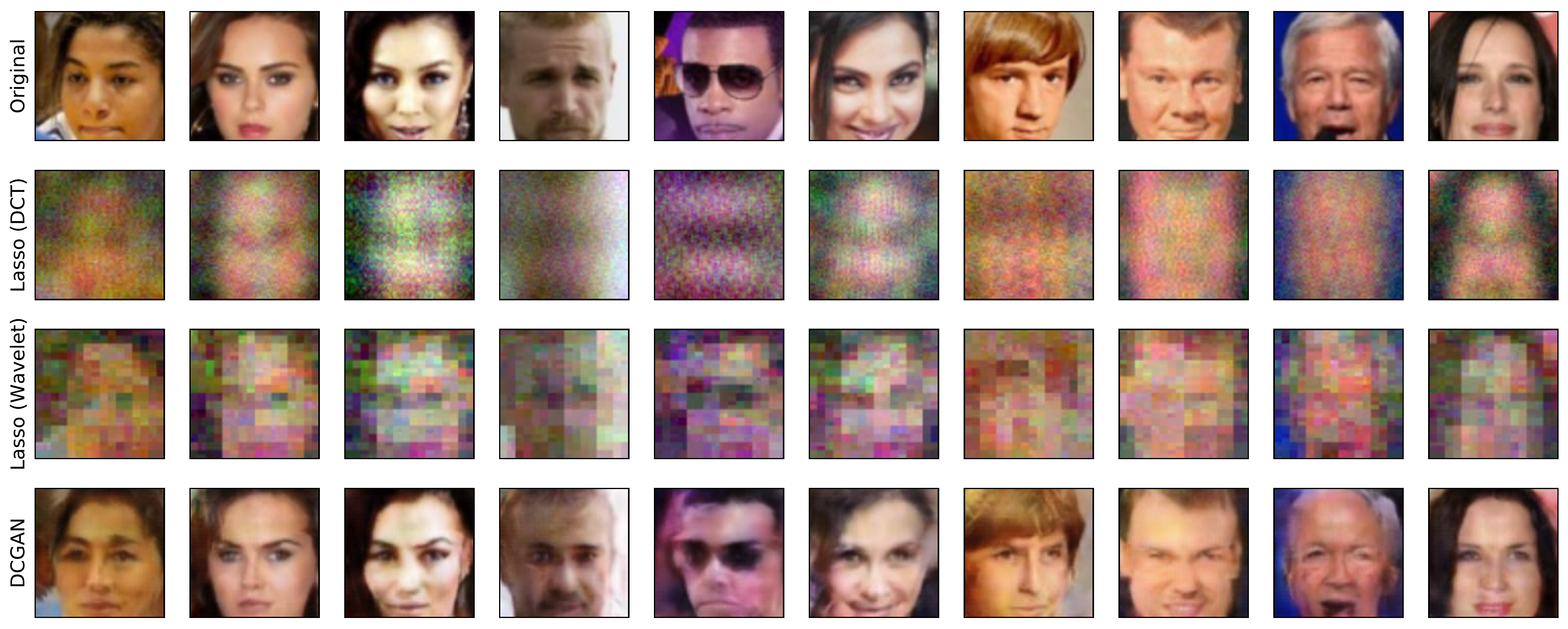}
    \caption{Reconstruction results on celebA with $m=500$ measurements (of $n=12288$ dimensional vector). We show original images (top row), and reconstructions by Lasso with DCT basis (second row), Lasso with wavelet basis (third row), and our algorithm (last row).}
\label{fig:celebA-reconstr}
\end{figure*}

\begin{figure*}
    \centering
    \includegraphics[width=0.9\textwidth]{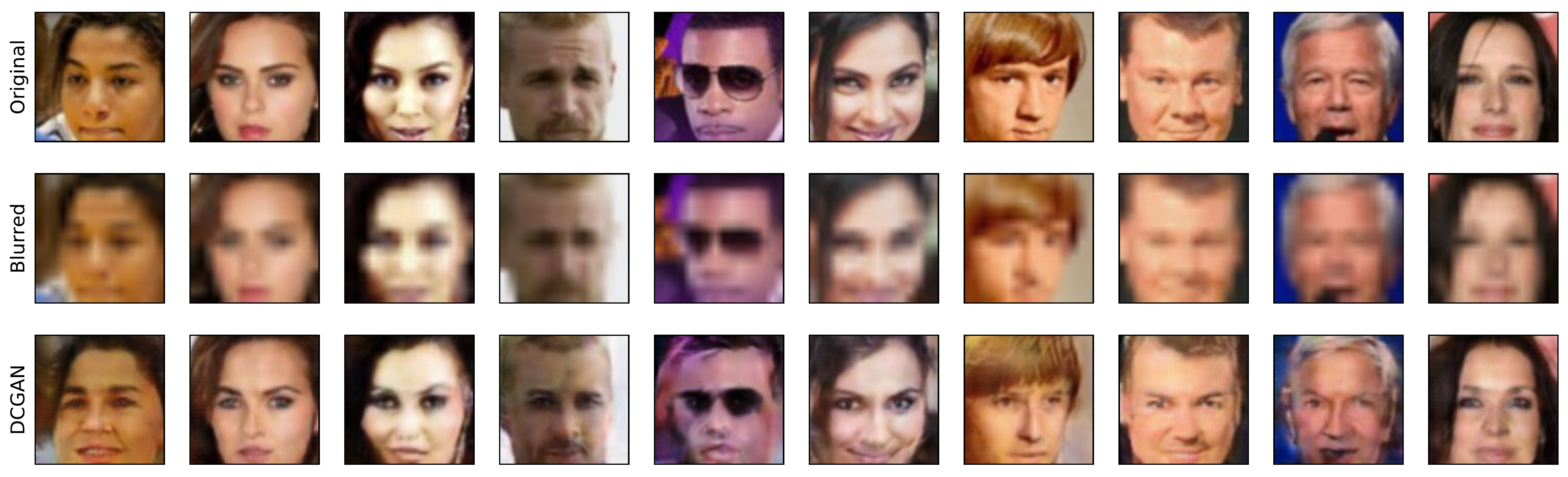}
    \caption{Super-resolution results on celebA. Top row has the original images. Second row shows the low resolution ($4x$ smaller) version of the original image. Last row shows the images produced by our algorithm.}
\label{fig:celebA-superres}
\end{figure*}

\begin{figure*}
    \centering
    \includegraphics[width=0.9\textwidth]{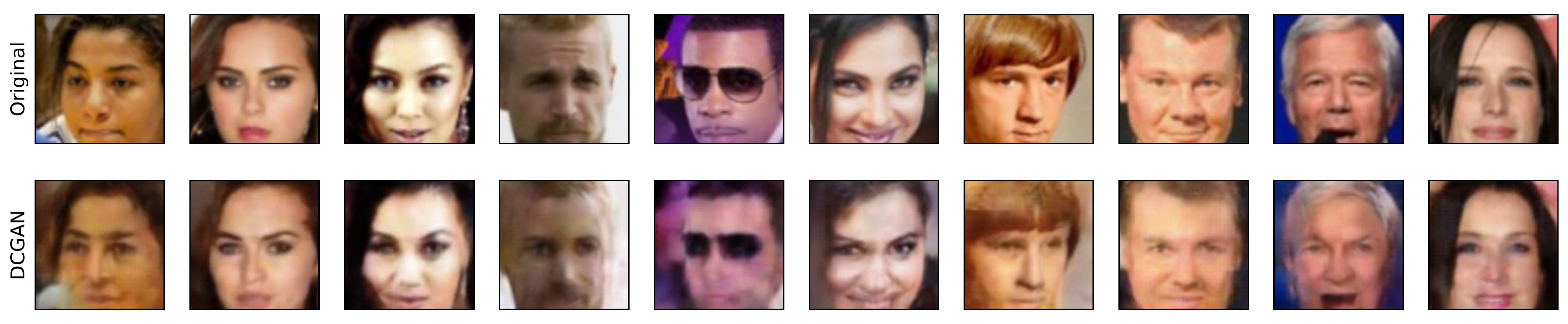}
    \caption{Results on the representation error experiments on celebA. Top row shows original images and the bottom row shows closest images found in the range of the generator.}
\label{fig:celebA-manifold}
\end{figure*}

\newsavebox\myboxtwo
\savebox{\myboxtwo}{\includegraphics[width=0.48\textwidth]{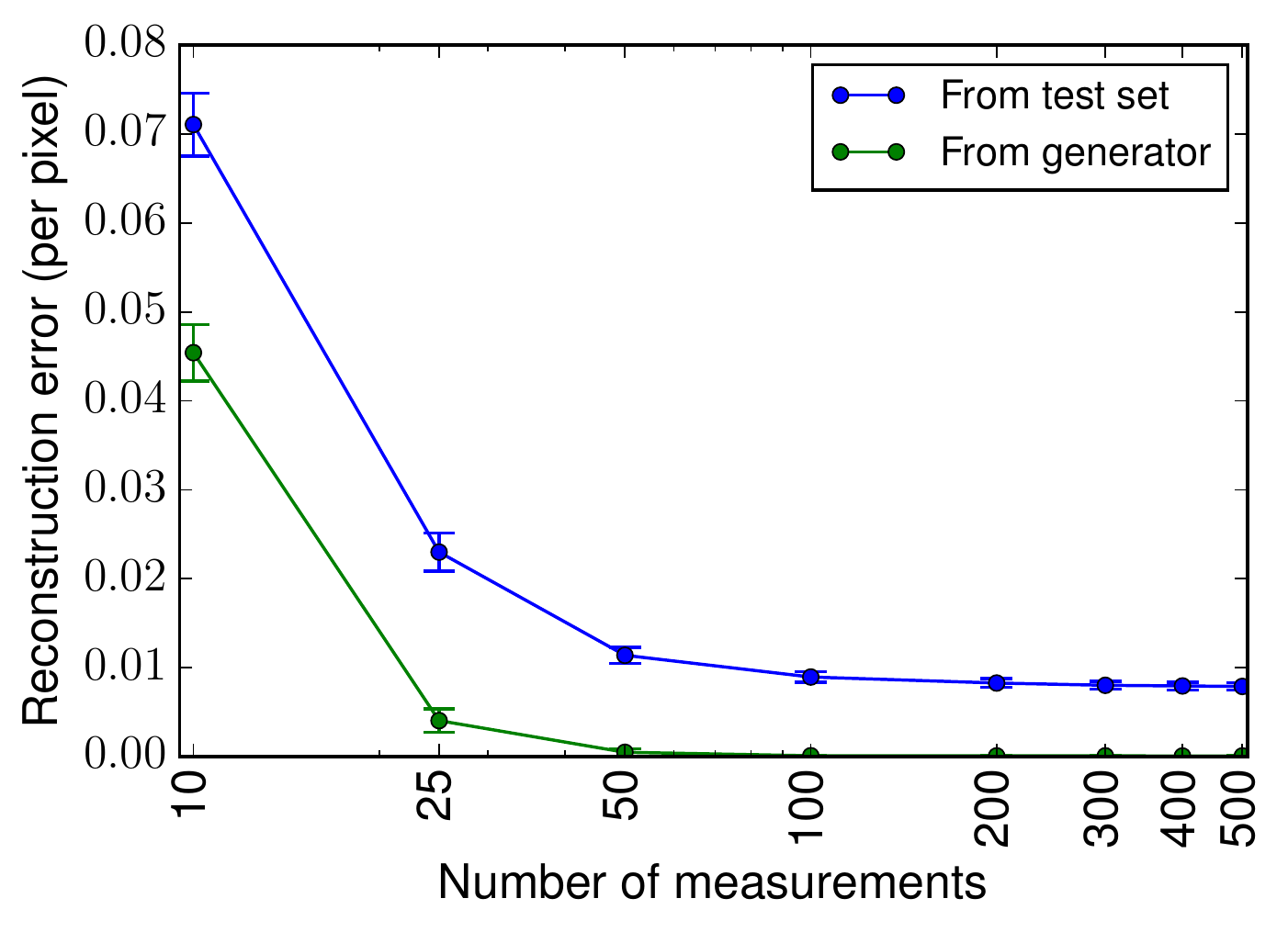}}
\begin{figure*}
    \begin{subfigure}[t]{0.48\textwidth}
        \usebox{\myboxtwo}
        \vspace*{-3mm}
        \caption{Results on MNIST}
    \label{fig:mnist-gen-range}
    \end{subfigure}\hfill%
    \begin{subfigure}[t]{0.48\textwidth}
        \vbox to \ht\myboxtwo{%
        \vfill
        \includegraphics[width=\textwidth]{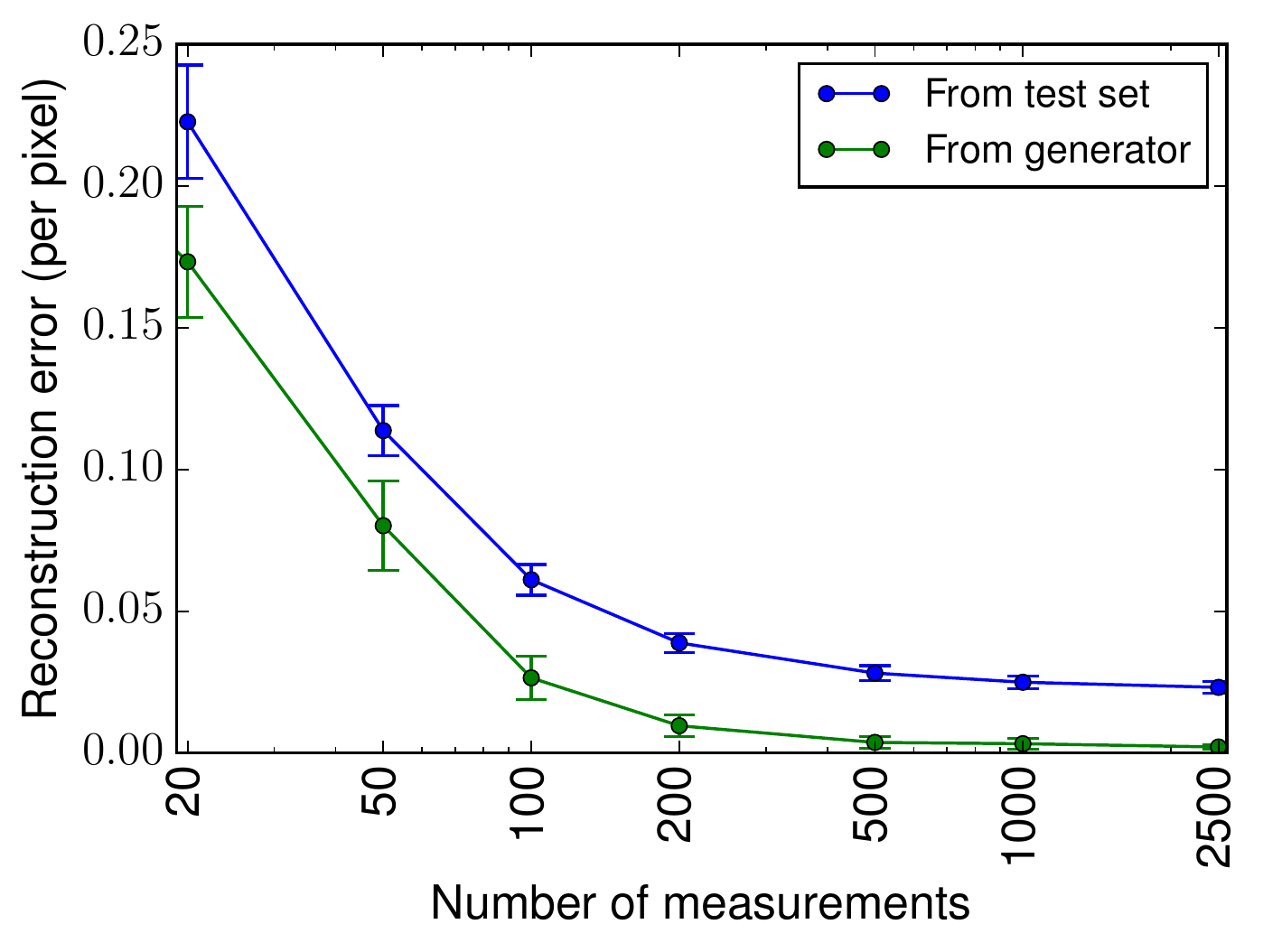}
        }
        \vspace*{2mm}
        \caption{Results on celebA}
    \label{fig:celebA-gen-range}
    \end{subfigure}
    \caption{Reconstruction error for images in the range of the generator. The vertical bars indicate 95\% confidence intervals.}
\label{fig:gen-range-l2}
\end{figure*}

\subsection{Understanding sources of error}

Although better than baselines, our reconstructions still admit some error. There are three sources of this error: (a) Representation error: the image being sensed is far from the range of the generator (b) Measurement error: The finite set of random measurements do not contain all the information about the unknown image (c) Optimization error: The optimization procedure did not find the best $z$.

In this section we present some experiments that suggest that the representation error is the dominant term. In our first experiment, we ensure that the representation error is zero, and try to minimize the sum of other two errors. In the second experiment, we ensure that the measurement error is zero, and try to minimize the sum of other two.

\begin{figure}
    \centering
    \includegraphics[width=0.5\textwidth]{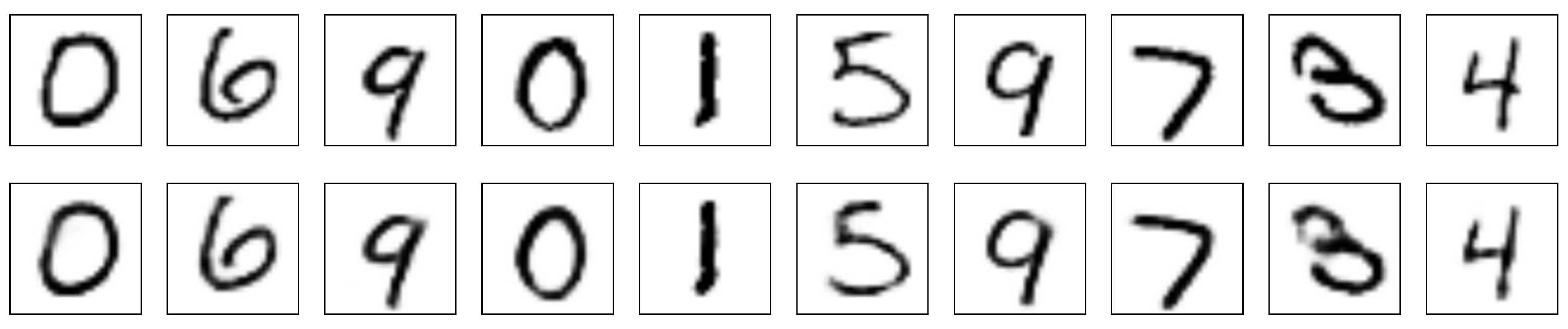}
    \caption{Results on the representation error experiments on MNIST\@. Top row shows original images and the bottom row shows closest images found in the range of the generator.}
\label{fig:mnist-manifold}
\vspace{-6mm}
\end{figure}

\subsubsection{Sensing images from the range of the generator}

Our first approach is to sense an image that \emph{is} in the range of the generator. More concretely, we sample a $z^*$ from $P_Z$. Then we pass it through the generator to get $x^* = G(z^*)$. Now, we pretend that this is a real image and try to sense that. This method eliminates the representation error and allows us to check if our gradient based optimization procedure is able to find $z^*$ by minimizing the objective.

In Fig.~\ref{fig:mnist-gen-range} and Fig.~\ref{fig:celebA-gen-range}, we show the reconstruction error for images in the range of the generators trained on MNIST and celebA datasets respectively. We see that we get almost perfect reconstruction with very few measurements. This suggests that objective is being properly minimized and we indeed get $\hat{z}$ close to $z^*$. \textit{i.e.} the sum of optimization error and the measurement error is not very large, in the absence of the representation error.

\subsubsection{Quantifying representation error}

We saw that in absence of the representation error, the overall error is small. However from Fig.~\ref{fig:reconstr-l2}, we know that the overall error is still non-zero. So, in this experiment, we seek to quantify the representation error, \textit{i.e.}, how far are the real images from the range of the generator?

From the previous experiment, we know that the $\hat{z}$ recovered by our algorithm is close to $z^*$, the best possible value, if the image being sensed is in the range of the generator. Based on this, we make an assumption that this property is also true for real images. With this assumption, we get an estimate to the representation error as follows: We sample real images from the test set. Then we use the full image in our algorithm, \textit{i.e.}, our measurement matrix $A$ is identity. This eliminates the measurement error. Using these measurements, we get the reconstructed image $G(\hat{z})$ through our algorithm. The estimated representation error is then $\|G(\hat{z}) - x^*\|^2$. We repeat this procedure several times over randomly sampled images from our dataset and report average representation error values. The task of finding the closest image in the range of the generator has been studied in prior work~\cite{creswell2016inverting, dumoulin2016adversarially, donahue2016adversarial}.

On the MNIST dataset, we get average per pixel representation error of $0.005$. The recovered images are shown in Fig.~\ref{fig:mnist-manifold}. In contrast with only $100$ Gaussian measurements, we are able to get a per pixel reconstruction error of about $0.009$.

On the celebA dataset, we get average per pixel representation error of $0.020$. The recovered images are shown in Fig.~\ref{fig:celebA-manifold}. On the other hand, with only $500$ Gaussian measurements, we get a per pixel reconstruction error of about $0.028$.

These experiments suggest that the representation error is the major component of the total error. Thus, a more flexible generative model can help to decrease the overall error on both datasets.

\section{Conclusion}

We demonstrate how to perform compressed sensing using generative models from neural nets.  These models can represent data distributions more concisely than standard sparsity models, while their differentiability allows for fast signal reconstruction.  This will allow compressed sensing applications to make significantly fewer measurements.

Our theorems and experiments both suggest that, after relatively few measurements, the signal reconstruction gets close to the optimal within the range of the generator.  To reach the full potential of this technique, one should use larger generative models as the number of measurements increase.  Whether this can be expressed more concisely than by training multiple independent generative models of different sizes is an open question.

Generative models are an active area of research with ongoing rapid improvements. Because our framework applies to general generative models, this improvement will immediately yield better reconstructions with fewer measurements. We also believe that one could also use the performance of generative models for our task as one benchmark for the quality of different models.

\clearpage
\section*{Acknowledgements}

We would like to thank Philipp Kr\"ahenb\"uhl for helpful discussions.
\bibliography{main}
\bibliographystyle{plain}
\clearpage
\newtheorem*{nonumlemma}{Lemma}

\section{Appendix A}

\begin{lemma}\label{lemma:rec_application}
    Given $S \subseteq \mathbb{R}^n$, $y \in \mathbb{R}^m$, $A \in \mathbb{R}^{m \times n}$, and $\gamma, \delta, \epsilon_1, \epsilon_2 > 0$, if matrix $A$ satisfies the $\SREC{}(S, \gamma, \delta)$, then for any two $x_1, x_2 \in S$, such that $\|Ax_1 - y\| \leq \epsilon_1$ and $\|Ax_2 - y\| \leq \epsilon_2$, we have
    \[
        \| x_1 - x_2 \| \leq \dfrac{\epsilon_1 + \epsilon_2 + \delta}{\gamma}.
    \]
\end{lemma}

\begin{proof}
    \begin{align*}
        \| x_1 - x_2 \| & \leq \dfrac{1}{\gamma} \left( \| Ax_1 - Ax_2 \| + \delta \right), \\
                        & = \dfrac{1}{\gamma} \left( \| (Ax_1 - y) - (Ax_2-y) \| + \delta \right), \\
                        & \leq \dfrac{1}{\gamma} \left( \| (Ax_1 - y) \| + \| (Ax_2-y) \| + \delta \right), \\
                        & \leq \dfrac{\epsilon_1 + \epsilon_2 + \delta}{\gamma}.
    \end{align*}
\end{proof}

\subsection{Proof of Lemma~\ref{thm:rec_with_p}}

\begin{definition}\label{subgamma}
    A random variable $X$ is said to be $\mathrm{subgamma}(\sigma, B)$ if \ $\forall \epsilon \geq 0$, we have
    \[
        \mathbb{P} \left( |X - \mathbb{E}[X] | \geq \epsilon \right) \leq 2 \max \left(e^{-\epsilon^2/(2\sigma^2)}, e^{-B\epsilon/2}\right).
    \]
\end{definition}

\begin{lemma}\label{lemma:chaining-eps-nets}
    Let $G: \mathbb{R}^k \rightarrow \mathbb{R}^n$ be an $L$-Lipschitz function. Let $B^k(r)$ be the $L_2$-ball in $\mathbb{R}^k$ with radius $r$, $S = G(B^k(r))$, and $M$ be a $\delta/L$-net on $B^k(r)$ such that $|M| \leq k\log \left( \dfrac{4Lr}{\delta} \right)$. Let $A$ be a $\mathbb{R}^{m \times n}$ random matrix with IID Gaussian entries with zero mean and variance $1/m$. If $$m = \Omega \left( k \log \dfrac{Lr}{\delta} \right),$$ then for any $x \in S$, if $x' = \argmin_{\wh{x} \in G(M)} \|x - \wh{x}\|$, we have $\|A(x - x')\| = \mathcal{O}(\delta)$ with probability $1 - e^{-\Omega(m)}$.
\end{lemma}

Note that for any given point $x'$ in $S$, if we try to find its nearest neighbor of that point in an $\delta$-net on $S$, then the difference between the two is at most the $\delta$. In words, this lemma says that even if we consider measurements made on these points, \textit{i.e.} a linear projection using a random matrix $A$, then as long as there are enough measurements, the difference between measurements is of the same order $\delta$. If the point $x'$ was in the net, then this can be easily achieved by Johnson-Lindenstrauss Lemma. But to argue that this is true for all $x'$ in $S$, which can be an uncountably large set, we construct a chain of nets on $S$. We now present the formal proof.

\begin{proof}

    Observe that $ \dfrac{\|Ax\|^2}{\|x\|^2}$ is $\mathrm{subgamma} \left(\dfrac{1}{\sqrt{m}}, \dfrac{1}{m} \right)$. Thus, for any $f > 0$,
    \[
        \epsilon \geq 2 +  \dfrac{4}{m} \log \dfrac{2}{f} \geq \max \left( \sqrt{\dfrac{2}{m} \log \dfrac{2}{f}} , \dfrac{2}{m} \log \dfrac{2}{f} \right)
    \]
    is sufficient to ensure that
    \[
        \mathbb{P} \left( \|Ax\| \geq (1+\epsilon) \|x\|  \right) \leq f.
    \]

    Now, let $M = M_0 \subseteq M_1 \subseteq M_2, \cdots \subseteq M_l$ be a chain of epsilon nets of $B^k(r)$ such that $M_i$ is a $\delta_i/L$-net and $\delta_i = \delta_0 / 2^i$, with $\delta_0 = \delta$. We know that there exist nets such that
    \[
        \log |M_i| \leq k \log \left( \dfrac{4Lr}{\delta_i} \right) \leq ik  + k\log \left( \dfrac{4Lr}{\delta_0} \right).
    \]
    Let $N_i = G(M_i)$. Then due to Lipschitzness of $G$, $N_i$'s form a chain of epsilon nets such that $N_i$ is a $\delta_i$-net of $S = G(B^k(r))$, with $|N_i| = |M_i|$.

    For $i \in \lbrace 0, 1, 2 \cdots, l-1 \rbrace $, let
    \[
        T_i = \lbrace x_{i+1} - x_i \mid x_{i+1} \in N_{i+1}, x_i \in N_i \rbrace.
    \]

    Thus,
    \begin{align*}
        |T_i|               &\leq |N_{i+1}||N_i|. \\
        \implies \log |T_i| &\leq \log |N_{i+1}| + | \log |N_i|, \\
                            &\leq (2i+1)k + 2k\log \left( \dfrac{4Lr}{\delta_0} \right), \\
                            &\leq 3ik + 2k\log \left( \dfrac{4Lr}{\delta_0} \right).
    \end{align*}

    Now assume $m = 3k\log \left( \dfrac{4Lr}{\delta_0} \right)$,
    \[
        \log(f_i) = - (m + 4ik),
    \]
    and
    \begin{align*}
        \epsilon_i &= 2 + \dfrac{4}{m} \log \dfrac{2}{f_i}, \\
                   &= 2 + \dfrac{4}{m} \log 2 + 4 + \dfrac{16ik}{m},  \\
                   &= O(1) + \dfrac{16ik}{m}.
    \end{align*}

    By choice of $f_i$ and $\epsilon_i$, we have $\forall i \in [l-1], \forall t \in T_i$,
    \[
        \mathbb{P} \left( \|At\| > (1+\epsilon_i) \|t\|  \right) \leq f_i.
    \]
    Thus by union bound, we have
    \[
        \mathbb{P} \left( \|At\| \leq (1+\epsilon_i) \|t\|, \forall i , \forall t \in T_i \right) \geq  1 - \sum_{i=0}^{l-1} |T_i| f_i.
    \]

    Now,
    \begin{align*}
        \log(|T_i|f_i) &= \log(|T_i|) + \log(f_i), \\
                       &\leq -k\log \left( \dfrac{4Lr}{\delta_0} \right) - ik, \\
                       &= -m/3 - ik. \\
        \implies \sum_{i=0}^{l-1}|T_i|f_i   &\leq e^{-m/3}\sum_{i=0}^{l-1}e^{-ik}, \\
                                            &\leq e^{-m/3}\left( \frac{1}{1-e^{-1}} \right), \\
                                            &\leq 2e^{-m/3}.
    \end{align*}

    Observe that for any $x \in S$, we can write
    \begin{align*}
        x           &= x_0 + (x_1 - x_0) + (x_2 - x_1) \ldots (x_l - x_{l-1}) + x^f. \\
        x - x_0     &= \sum_{i=0}^{l-1}(x_{i+1} - x_i) + x^f.
    \end{align*}
    where $x_i \in N_i$ and $x_f = x - x_l$.

    Since each $x_{i+1} - x_i \in T_i$, with probability at least $1 - 2e^{-m/3}$, we have
    \begin{align*}
    \sum_{i=0}^{l-1} \| A(x_{i+1} - x_i) \|  &= \sum_{i=0}^{l-1} (1 + \epsilon_i) \|(x_{i+1} - x_i) \|, \\
                                             &\leq \sum_{i=0}^{l-1} (1 + \epsilon_i) \delta_i, \\
                                             &= \delta_0 \sum_{i=0}^{l-1} \dfrac{1}{2^i} \left(O(1) + \dfrac{16ik}{m} \right), \\
                                             &= O(\delta_0) + \delta_0 \dfrac{16k}{m} \sum_{i=0}^{l-1} \left(\dfrac{i}{2^i} \right), \\
                                             &= O(\delta_0).
    \end{align*}

    Now, $\|x^f\| = \|x - x_l\| \leq d_l = \dfrac{\delta_0}{2_l}$, and $\| x_{i+1} - x_{i} \| \leq \delta_i$ due to properties of epsilon-nets. We know that $\|A\| \leq 2 + \sqrt{n/m}$ with probability at least $1 - 2e^{-m/2}$ (Corollary 5.35~\cite{vershynin2010introduction}). By setting $l = \log(n)$, we get that, $\|A\| \|x^f\| \leq \left(2 + \sqrt{\dfrac{n}{m}} \right) \dfrac{\delta_0}{2^l} = O(\delta_0)$ with probability $\geq 1 - 2e^{-m/2}$.

    Combining these two results, and noting that it is possible to choose $x' = x_0$, we get that with probability $1 - e^{-\Omega(m)}$,
    \begin{align*}
        \|A(x - x')\| &= \|A(x - x_0)\|, \\
                      &\leq \sum_{i=0}^{l-1}\| A(x_{i+1} - x_i)\| + \| Ax^f \|, \\
                      &= \mathcal{O}(\delta_0) + \| A \| \|x^f \|, \\
                      &= \mathcal{O}(\delta).
    \end{align*}

\end{proof}

\begin{nonumlemma}
    Let $G:\mathbb{R}^k \rightarrow \mathbb{R}^n$ be $L$-Lipschitz. Let $$B^k(r) = \lbrace z \ \vert \ z \in \mathbb{R}^k, \| z \| \leq r \rbrace$$ be an $L_2$-norm ball in $\mathbb{R}^k$. For $\alpha < 1$, if
    \[
        m = \Omega \left(\frac{k}{\alpha^2} \log \frac{Lr}{\delta} \right),
    \]
    then a random matrix $A \in \mathbb{R}^{m\times n}$ with IID entries such that $A_{ij} \sim \mathcal{N}\left(0,\frac{1}{m}\right)$ satisfies the $\SREC{}(G(B^k(r)), 1-\alpha, \delta)$ with $1 - e^{-\Omega(\alpha^2 m)}$ probability.
\end{nonumlemma}

\begin{proof}
    We construct a $\dfrac{\delta}{L}$-net, $N$, on $B^k(r)$. There exists a net such that
    \[
        \log |N| \leq k \log \left( \frac{4Lr}{\delta} \right).
    \]

    Since $N$ is a $\dfrac{\delta}{L}$-cover of $B^k(r)$, due to the $L$-Lipschitz property of $G(\cdot)$, we get that $G(N)$ is a $\delta$-cover of $G(B^k(r))$.

    Let $T$ denote the pairwise differences between the elements in $G(N)$, \textit{i.e.},
    \[
        T = \lbrace G(z_1) - G(z_2) \mid z_1,z_2\in N \rbrace.
    \]

    Then,
    \begin{align*}
        |T|                 & \leq |N|^2, \\
        \implies \log |T|   & \leq 2 \log |N|, \\
                            & \leq 2k \log \left( \frac{4Lr}{\delta} \right).
    \end{align*}

    For any $z, z' \in B^k$, $\exists \ z_1, z_2\in N$, such that $G(z_1), G(z_2)$ are $\delta$-close to $G(z)$ and $G(z')$ respectively. Thus, by triangle inequality,
    \begin{align*}
        \|G(z) - G(z')\|    & \leq      \|G(z) - G(z_1) \| + \\
                            & \qquad    \|G(z_1) - G(z_2) \| + \\
                            & \qquad    \|G(z_2) - G(z') \|, \\
                            & \leq      \|G(z_1) - G(z_2) \| + 2 \delta.
    \end{align*}

    Again by triangle inequality,
    \begin{align*}
        \|AG(z_1) - AG(z_2)\|   &\leq       \|AG(z_1) - AG(z)\| + \\
                                &\qquad     \|AG(z) - AG(z')\| + \\
                                &\qquad     \|AG(z') - AG(z_2)\|.
    \end{align*}

    Now, by Lemma~\ref{lemma:chaining-eps-nets}, with probability $1 - e^{-\Omega(m)}$, $\|AG(z_1) - AG(z)\| = \mathcal{O}(\delta)$, and $\|AG(z') - AG(z_2)\| = \mathcal{O}(\delta)$.
    Thus,
    \[
        \|AG(z_1) - AG(z_2)\| \leq \|AG(z) - AG(z')\| + \mathcal{O}(\delta).
    \]

    By the Johnson-Lindenstrauss Lemma, for a fixed $x \in \mathbb{R}^n$, $\mathbb{P}\left[\|Ax\|^2 < (1-\alpha)\|x\|^2 \right]< \exp(-\alpha^2 m)$. Therefore, we can union bound over all vectors in $T$ to get
    \[
        \mathbb{P} (\|Ax\|^2 \geq (1-\alpha)\|x\|^2, \ \forall x \in T) \geq 1- e^{-\Omega(\alpha^2 m)}.
    \]

    Since $\alpha < 1$, and $z_1, z_2 \in N$, $G(z_1) - G(z_2) \in T$, we have
    \begin{align*}
        (1-\alpha) \| G(z_1) - G(z_2) \|    &\leq \sqrt{1-\alpha} \| G(z_1) - G(z_2) \|, \\
                                            &\leq \|A G(z_1) - A G(z_2)\|.
    \end{align*}

    Combining the three results above we get that with probability $1- e^{-\Omega(\alpha^2 m)}$,
    \begin{align*}
        (1-\alpha) \|G(z) - G(z')\| &\leq  (1-\alpha) \|G(z_1) - G(z_2) \| + \mathcal{O}(\delta), \\
                                    &\leq \|A G(z_1) - A G(z_2)\| + \mathcal{O}(\delta), \\
                                    &\leq  \|AG(z) - AG(z')\| + \mathcal{O}(\delta).
    \end{align*}

    Thus, $A$ satisfies $\SREC{}(S, 1-\alpha, \delta)$ with probability  $1- e^{-\Omega(\alpha^2 m)}$.

\end{proof}

\subsection{Proof of Lemma~\ref{thm: subspace embedding}}

\begin{lemma}\label{lemma:counting-partitions}
	Consider $c$ different $k-1$ dimensional hyperplanes in $\mathbb{R}^k$. Consider the $k$-dimensional faces (hereafter called $k$-faces) generated by the hyperplanes, \textit{i.e.} the elements in the partition of $\mathbb{R}^k$ such that relative to each hyperplane, all points inside a partition are on the same side. Then, the number of $k$-faces is $\mathcal{O}(c^k)$.
\end{lemma}

\begin{proof}
    Proof is by induction, and follows~\cite{matouvsek2002lectures}.

    Let $f(c,k)$ denote the number of $k-$faces generated in $\mathbb{R}^k$ by $c$ different $(k-1)$-dimensional hyperplanes. As a base case, let $k=1$. Then $(k-1)$-dimensional hyperplanes are just points on a line. $c$ points partition $\mathbb{R}$ into $c+1$ pieces. This gives $f(c,1) = \mathcal{O}(c)$.

    Now, assuming that $f(c,k-1) = \mathcal{O}(c^{k-1})$ is true, we need to show $f(c,k) = \mathcal{O}(c^k)$. Assume we have $(c-1)$ different hyperplanes $H=\{h_1, h_2, \ldots, h_{c-1}\}\subset\mathbb{R}^k$, and a new hyperplane $h_c$ is added. $h_c$ intersects $H$ at $(c-1)$ different $(k-2)$-faces given by $F=\{f_j \mid f_j = h_j \cap h_c, 1\leq j\leq (c-1)\}$. The $(k-2)$-faces in $F$ partition $h_c$ into $f(c-1, k-1)$ different $(k-1)$-faces. Additionally, each $(k-1)$-face in $h_c$ divides an existing $k$-face into two. Hence the number of new $k$-faces introduced by the addition of $h_c$ is $f(c-1, k-1)$. This gives the recursion
    \begin{align*}
        f(c,k)  &= f(c-1,k) + f(c-1,k-1), \\
                &= f(c-1,k) + \mathcal{O}(c^{k-1}), \\
                &= \mathcal{O}(c^k).
    \end{align*}

\end{proof}

\begin{nonumlemma}
    Let $G:\mathbb{R}^k \rightarrow \mathbb{R}^n$ be a $d$-layer neural network, where each layer is a linear transformation followed by a pointwise non-linearity. Suppose there are at most $c$ nodes per layer, and the non-linearities are piecewise linear with at most two pieces, and let
    \[
        m = \Omega \left( \frac{1}{\alpha^2} kd\log c \right)
    \]
    for some $\alpha < 1$.  Then a random matrix $A\in\mathbb{R}^{m\times n}$ with IID entries $A_{ij}\sim\mathcal{N}(0,\frac{1}{m})$ satisfies the $\SREC{}(G(\mathbb{R}^k), 1 - \alpha, 0)$ with $1 - e^{-\Omega(\alpha^2 m)}$ probability.
\end{nonumlemma}

\begin{proof}
    Consider the first layer of $G$. Each node in this layer can be represented as a hyperplane in $\mathbb{R}^k$, where the points on the hyperplane are those where the input to the node switches from one linear piece to the other. Since there are at most $c$ nodes in this layer, by Lemma~\ref{lemma:counting-partitions}, the input space is partitioned by at most $c$ different hyperplanes, into $\mathcal{O}(c^k)$ $k$-faces. Applying this over the $d$ layers of $G$, we get that the input space $\mathbb{R}^k$ is partitioned into at most $c^{kd}$ sets.

    Recall that the non-linearities are piecewise linear, and the partition boundaries were made precisely at those points where the non-linearities change from one piece to another. This means that within each set of the input partition, the output is a linear function of the inputs. Thus $G(\mathbb{R}^k)$ is a union of $c^{kd}$ different $k$-faces in $\mathbb{R}^n$.

    We now use an oblivious subspace embedding to bound the number of measurements required to embed the range of $G(\cdot)$. For a single $k$-face $S \subseteq \mathbb{R}^n$, a random matrix $A\in\mathbb{R}^{m\times n}$ with IID entries such that $A_{ij} \sim \mathcal{N}\left(0,\frac{1}{m}\right)$ satisfies $\SREC{}(S, 1-\alpha, 0)$ with probability $1 - e^{-\Omega(\alpha^2 m)}$ if $m = \Omega(k/\alpha^2)$.

    Since the range of $G(\cdot)$ is a union of $c^{kd}$ different $k$-faces, we can union bound over all of them, such that $A$ satisfies the $\SREC{}(G(\mathbb{R}^k),1 - \alpha, 0)$ with probability $1 - c^{kd}e^{-\Omega(\alpha^2 m)}$. Thus, we get that $A$ satisfies the $\SREC{}(G(\mathbb{R}^k), 1-\alpha, 0)$ with probability $1 - e^{-\Omega(\alpha^2 m)}$ if
    \[
        m = \Omega \left( \frac{kd\log c}{\alpha^2} \right).
    \]
\end{proof}

\subsection{Proof of Lemma~\ref{thm:rec_application}}

\begin{nonumlemma}
    Let $A \in \mathbb{R}^{m \times n}$ by drawn from a distribution that (1) satisfies the $\SREC{}(S, \gamma, \delta)$ with probability $1-p$ and (2) has for every fixed $x \in \R^n$, $\norm{Ax} \leq 2\norm{x}$ with probability $1-p$. For any $x^* \in \R^n$ and noise $\eta$, let $y = Ax^* + \eta$.  Let $\wh{x}$ approximately minimize $\norm{y - Ax}$ over $x \in S$, \textit{i.e.},
    \[
        \norm{y - A\wh{x}} \leq \min_{x \in S}\norm{y - Ax} + \eps.
    \]
    Then
    \[
        \norm{\wh{x} - x^*} \leq \left( \frac{4}{\gamma} + 1 \right) \min_{x \in S} \norm{x^* - x} + \frac{1}{\gamma}\left(2\norm{\eta} + \eps + \delta\right)
    \]
    with probability $1-2p$.
\end{nonumlemma}

\begin{proof}
    Let $\overline{x} = \argmin_{x \in S} \norm{x^* - x}$.  Then we have by Lemma~\ref{lemma:rec_application} and the hypothesis on $\wh{x}$ that
    \begin{align*}
        \norm{\overline{x} - \wh{x}} &\leq \frac{\norm{A\overline{x} - y} + \norm{A \wh{x}-y} + \delta}{\gamma}, \\
                                     &\leq \frac{2\norm{A\overline{x} - y} + \eps + \delta}{\gamma}, \\
                                     &\leq \frac{2\norm{A(\overline{x} - x^*)} + 2\norm{\eta} + \eps + \delta}{\gamma},
    \end{align*}
    as long as $A$ satisfies the \SREC{}, as happens with probability $1-p$.  Now, since $\overline{x}$ and $x^*$ are independent of $A$, by assumption we also have $\norm{A(\overline{x} - x^*)} \leq 2\norm{\overline{x}-x^*}$ with probability $1-p$. Therefore
    \[
        \norm{x^* - \wh{x}} \leq \norm{\overline{x} - x^*} + \frac{4\norm{\overline{x} - x^*} + 2\norm{\eta} + \eps + \delta}{\gamma}
    \]
    as desired.
\end{proof}

\subsection{Lipschitzness of Neural Networks}

\begin{lemma}\label{lemma:composition_lipschitz}
    Consider any two functions $f$ and $g$. If $f$ is $L_f$-Lipschitz and $g$ is $L_g$-Lipschitz, then their composition $f \circ g$ is $L_f L_g$-Lipschitz.
\end{lemma}
\begin{proof}
    For any two $x_1, x_2$,
    \begin{align*}
        \| f(g(x_1)) - f(g(x_2)) \| &\leq L_f \| g(x_1) - g(x_2) \|, \\
                                    &\leq L_f L_g \|x_1 - x_2 \|.
    \end{align*}
\end{proof}

\begin{lemma}\label{lemma:lipschitz}
    If G is a $d$-layer neural network with at most $c$ nodes per layer, all weights $\leq w_{\max}$ in absolute value, and $M$-Lipschitz non-linearity after each layer, then $G(\cdot)$ is $L$-Lipschitz with $L = {(Mcw_{\max})}^d$.
\end{lemma}
\begin{proof}
    Consider any linear layer with input $x$, weight matrix $W$ and bias vector $b$. Thus, $f(x) = Wx + b$. Now for any two $x_1$, $x_2$,
    \begin{align*}
        \|f(x_1) - f(x_2)\| &= \|Wx_1 + b -  Wx_2 + b\|, \\
                            &= \|W(x_1 - x_2)\|, \\
                            &\leq \|W\| \|(x_1 - x_2)\|, \\
                            &\leq c w_{\max} \|(x_1 - x_2)\|.
    \end{align*}
    Let $f_i(\cdot), i \in [d]$ denote the function for the $i$-th layer in $G$. Since each layer is a composition of a linear function and a non-linearity, by Lemma~\ref{lemma:composition_lipschitz}, have that $f_i$ is $Mcw_{\max}$-Lipschitz.

    Since $G = f_1 \circ f_2 \circ \ldots f_d$, by repeated application of Lemma~\ref{lemma:composition_lipschitz}, we get that $G$ is $L$-Lipschitz with $L = {(Mcw_{\max})}^d$.
\end{proof}

\section{Appendix B}

\newsavebox\myboxthree
\savebox{\myboxthree}{\includegraphics[width=0.48\textwidth]{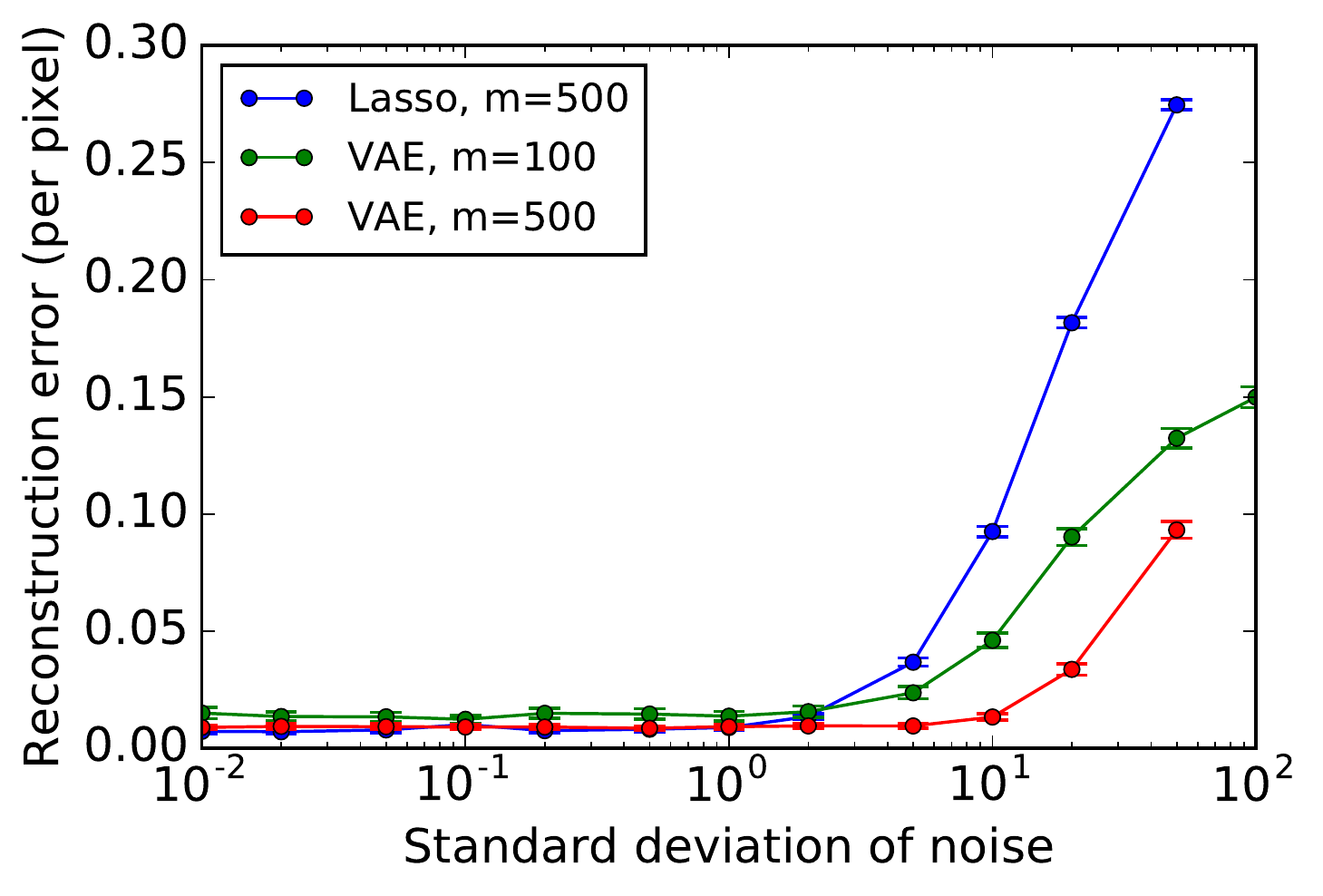}}
\begin{figure*}[tbph!]
    \begin{subfigure}[]{0.48\textwidth}
        \usebox{\myboxthree}
        \vspace*{-3mm}
        \caption{Results on MNIST.}
    \label{fig:mnist-noise-l2}
    \end{subfigure}\hfill%
    \begin{subfigure}[]{0.48\textwidth}
        \vbox to \ht\myboxthree{%
        \vfill
        \includegraphics[width=\textwidth]{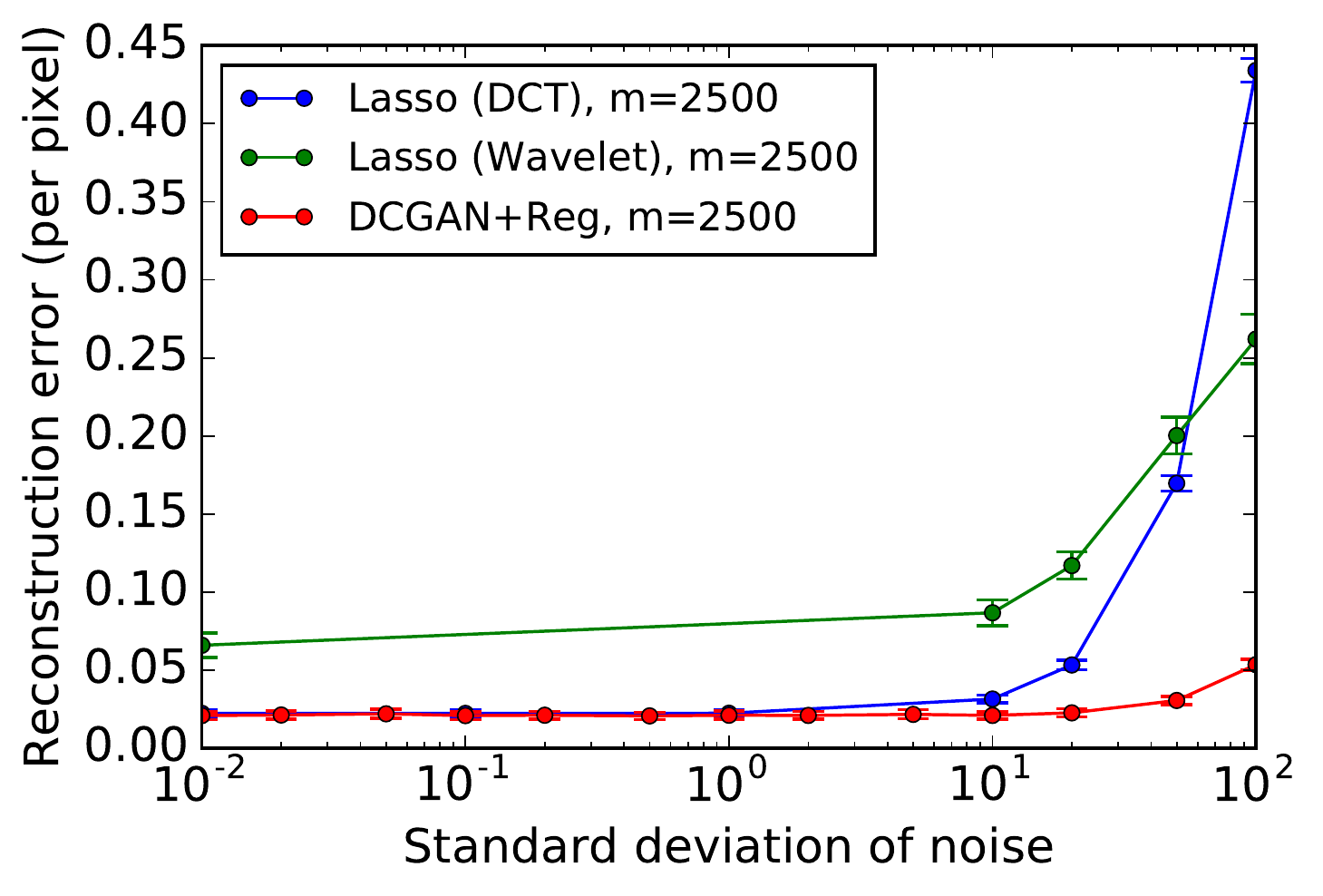}
        }
        \vspace*{2mm}
        \caption{Results on celebA.}
        \label{fig:celebA-noise-l2}
    \end{subfigure}
    \caption{Noise tolerance. We show a plot of per pixel reconstruction error as we vary the noise level ($\sqrt{\mathbb{E}[ {\| \eta \|}^2]}$). The vertical bars indicate 95\% confidence intervals.}
\label{fig:noise-l2}
\end{figure*}

\subsection{Noise tolerance}

To understand the noise tolerance of our algorithm, we do the following experiment: First we fix the number of measurements so that Lasso does as well as our algorithm. From Fig.~\ref{fig:mnist-reconstr-l2}, and Fig.~\ref{fig:celebA-reconstr-l2} we see that this point is at $m=500$ for MNIST and $m=2500$ for celebA. Now, we look at the performance as the noise level increases. Hyperparameters are kept fixed as we change the noise level for both Lasso and for our algorithm.

In Fig.~\ref{fig:mnist-noise-l2}, we show the results on the MNIST dataset. In Fig.~\ref{fig:mnist-noise-l2}, we show the results on celebA dataset. We observe that our algorithm has more noise tolerance than Lasso.

\begin{figure}[t]
    \centering
    \includegraphics[width=0.48\textwidth]{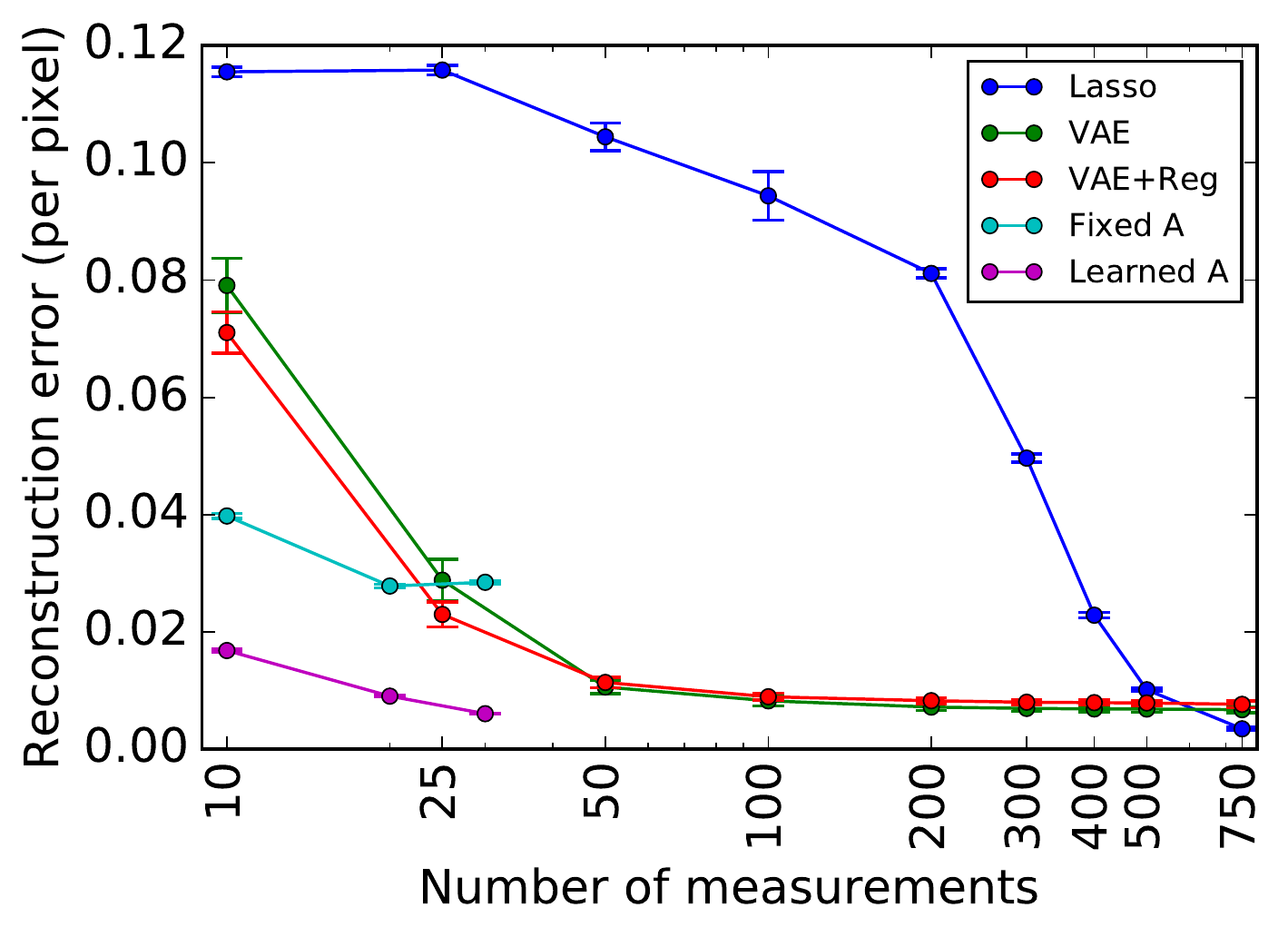}
    \caption{Results on End to End model on MNIST\@. We show per pixel reconstruction error vs number of measurements. `Fixed A' and `Learned A' are two end to end models. The end to end models get noiseless measurements, while the other models get noisy ones. The vertical bars indicate 95\% confidence intervals.}
\label{fig:mnist-reconstr-l2-all}
\end{figure}

\subsection{Other models}

\subsubsection{End to end training on MNIST}

Instead of using a generative model to reconstruct the image, another approach is to learn from scratch a mapping that takes the measurements and outputs the original image. A major drawback of this approach is that it necessitates learning a new network if get a different set of measurements.

If we use a random matrix for every new image, the input to the network is essentially noise, and the network does not learn at all. Instead we are forced to use a fixed measurement matrix. We explore two approaches. First is to randomly sample and fix the measurement matrix and learn the rest of the mapping. In the second approach, we jointly optimize the measurement matrix as well.

We do this for $10$, $20$ and $30$ measurements for the MNIST dataset. We did not use additive noise. The reconstruction errors are shown in Fig.~\ref{fig:mnist-reconstr-l2-all}. The reconstructions can be seen in Fig.~\ref{fig:mnist-e2e}.

\begin{figure*}
    \centering
    \includegraphics[width=0.9\textwidth]{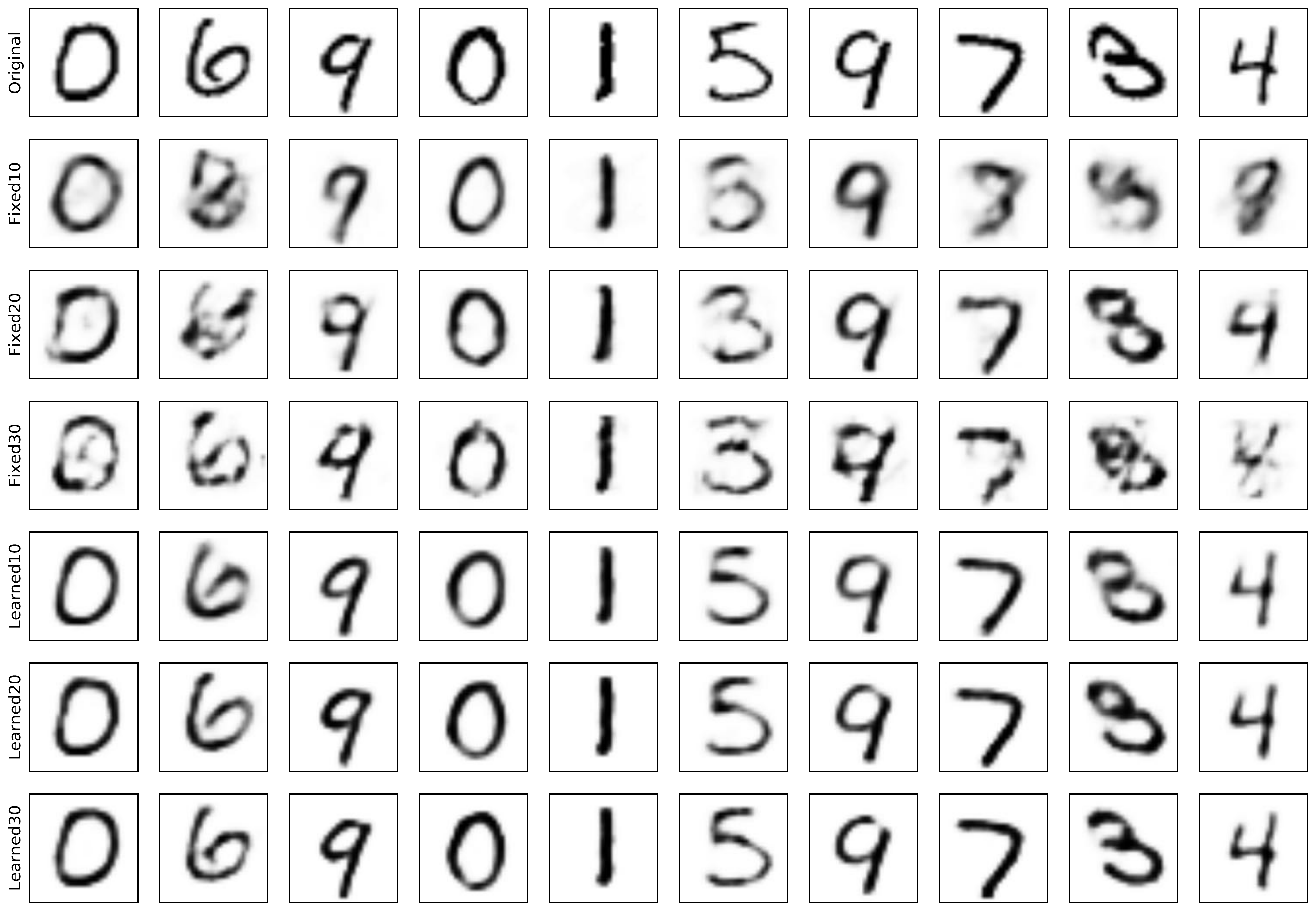}
    \caption{MNIST End to end learned model. Top row are original images. The next three are recovered by model with fixed random $A$, with 10, 20 and 30 measurements. Bottom three rows are with learned $A$ and $10$, $20$ and $30$ measurements.}
\label{fig:mnist-e2e}
\end{figure*}

\subsection{More results}

Here, we show more results on the reconstruction task, with varying number of measurements on both MNIST and celebA. Fig.~\ref{fig:more-mnist-reconstr} shows reconstructions on MNIST with $25$, $100$ and $400$ measurements. Fig.~\ref{fig:more-celebA-reconstr1}, Fig.~\ref{fig:more-celebA-reconstr2} and Fig.~\ref{fig:more-celebA-reconstr3} show results on celebA dataset.

\begin{figure*}
    \centering
    \begin{subfigure}[]{0.9\textwidth}
        \includegraphics[width=\textwidth]{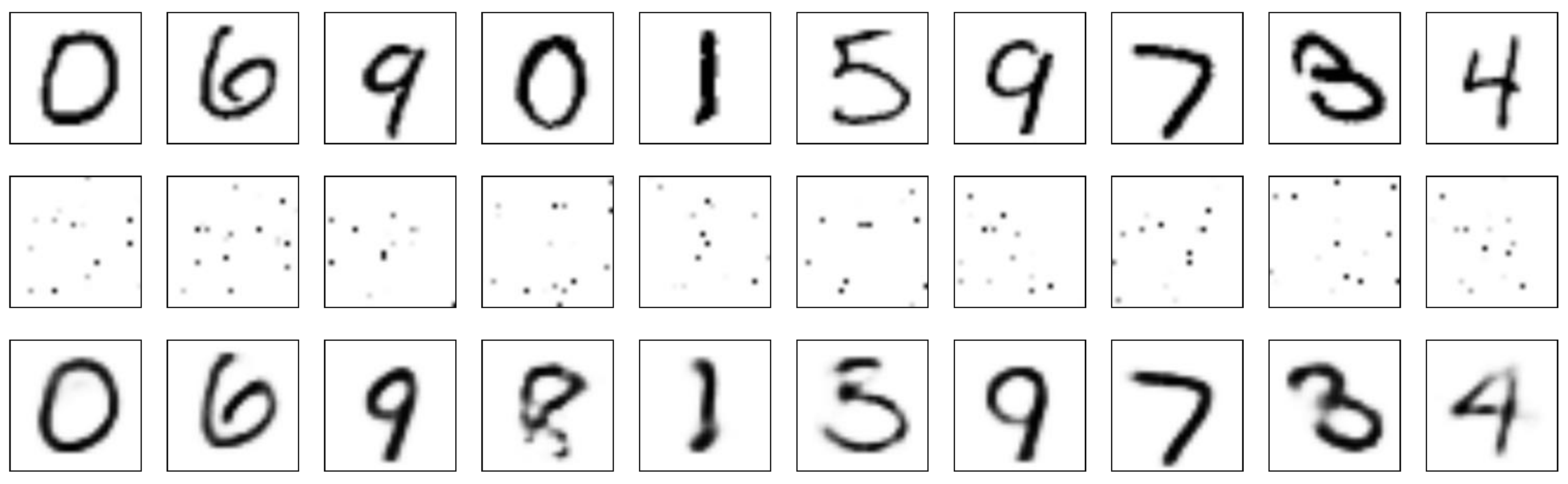}
        \caption{25 measurements}
    \end{subfigure}
    \begin{subfigure}[]{0.9\textwidth}
        \includegraphics[width=\textwidth]{mnist_reconstr_100_orig_lasso_vae-gen.pdf}
        \caption{100 measurements}
    \end{subfigure}
    \begin{subfigure}[]{0.9\textwidth}
        \includegraphics[width=\textwidth]{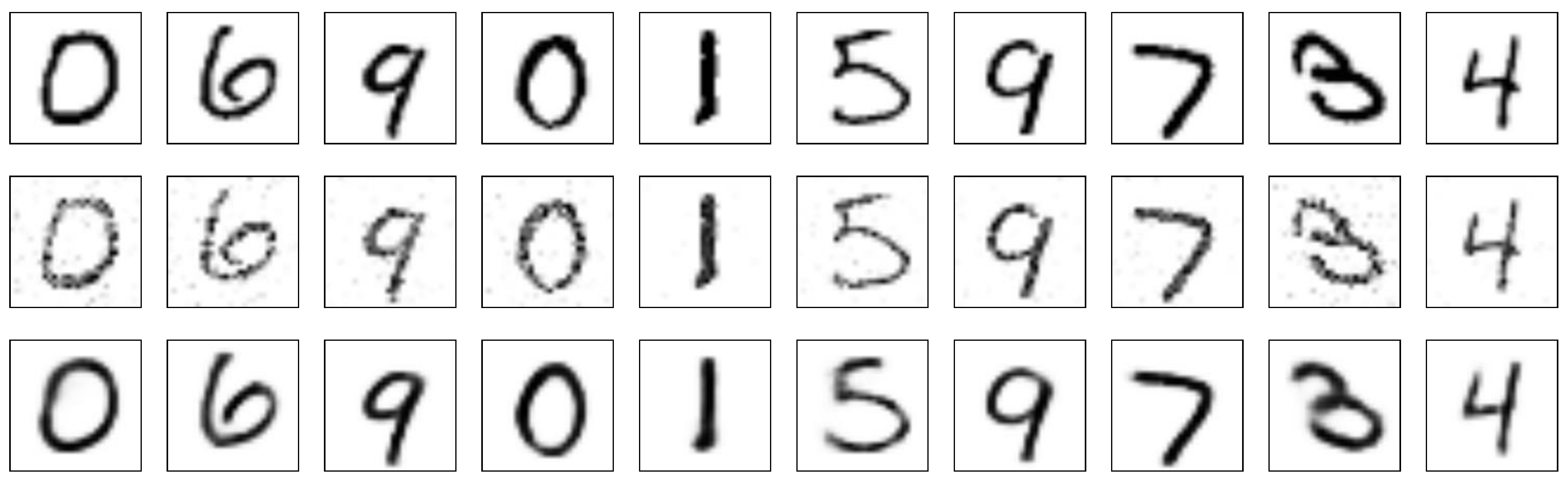}
        \caption{400 measurements}
    \end{subfigure}
    \caption{Reconstruction on MNIST\@. In each image, top row is ground truth, middle row is Lasso, bottom row is our algorithm.}
\label{fig:more-mnist-reconstr}
\end{figure*}

\begin{figure*}
    \begin{center}
        \begin{subfigure}[]{0.8\textwidth}
            \includegraphics[width=\textwidth]{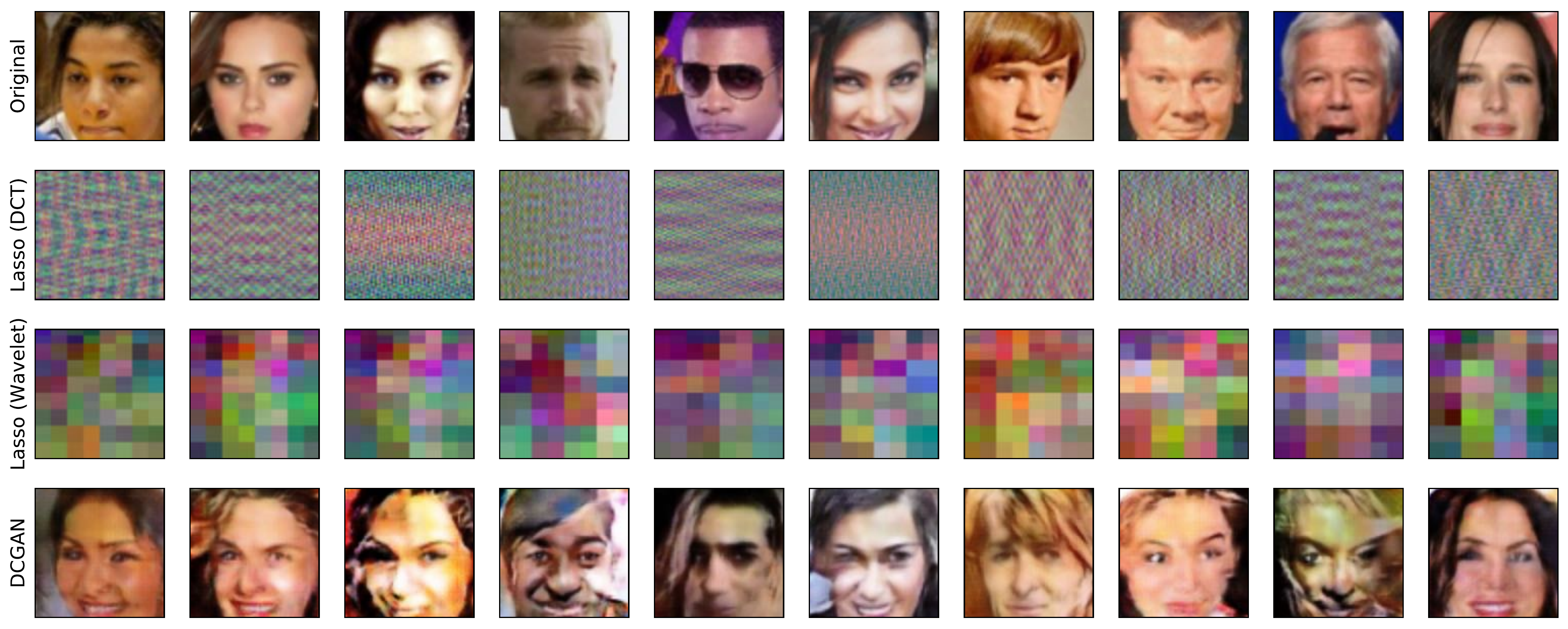}
            \caption{50 measurements}
        \end{subfigure}
        \begin{subfigure}[]{0.8\textwidth}
            \includegraphics[width=\textwidth]{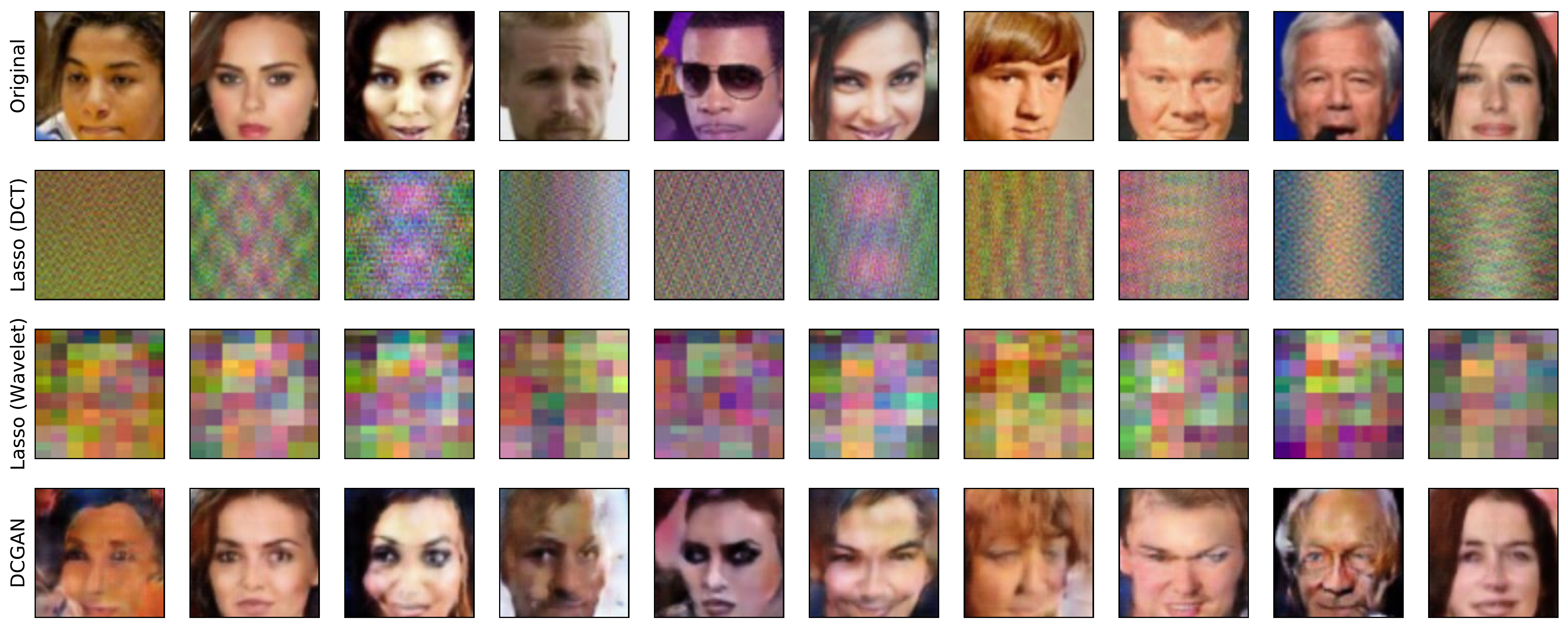}
            \caption{100 measurements}
        \end{subfigure}
        \begin{subfigure}[]{0.8\textwidth}
            \includegraphics[width=\textwidth]{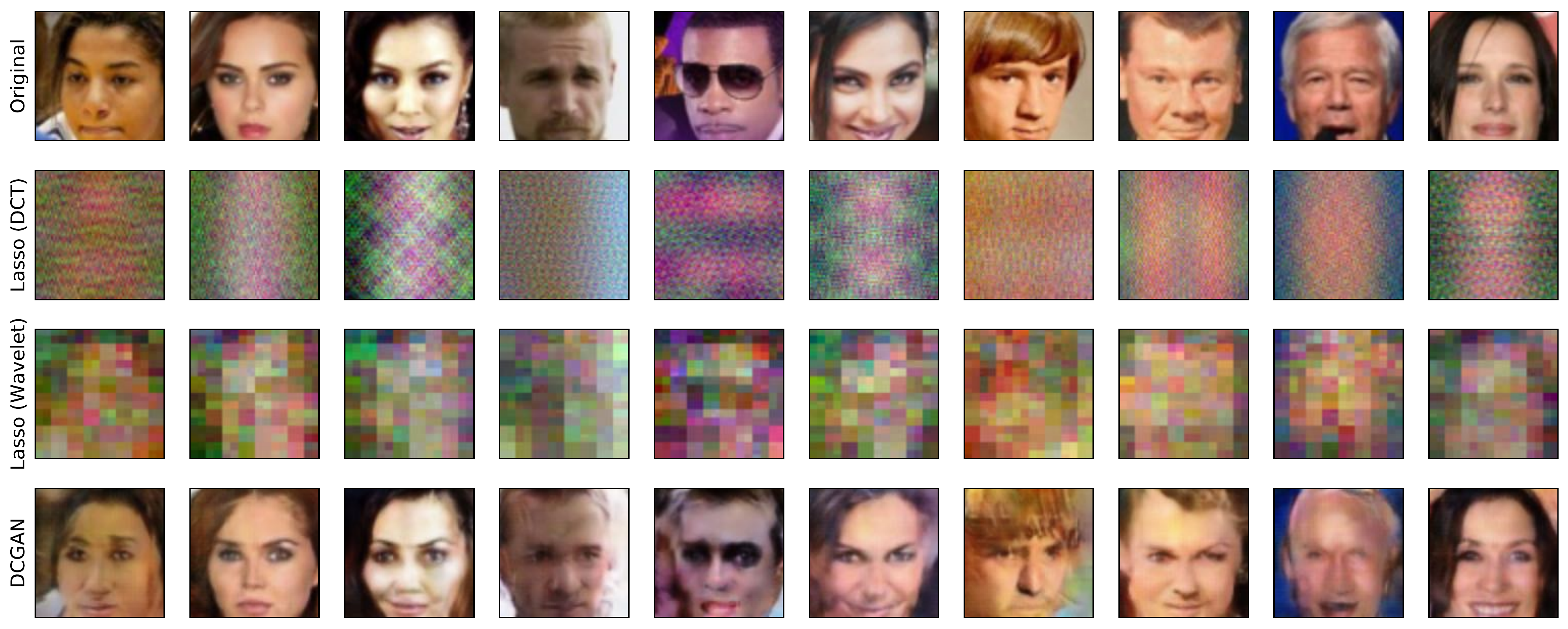}
            \caption{200 measurements}
        \end{subfigure}
        \caption{Reconstruction on celebA. In each image, top row is ground truth, subsequent two rows show reconstructions by Lasso (DCT) and Lasso (Wavelet) respectively. The bottom row is the reconstruction by our algorithm.}
		\label{fig:more-celebA-reconstr1}
    \end{center}
\end{figure*}

\begin{figure*}
    \begin{center}
        \begin{subfigure}[]{0.8\textwidth}
            \includegraphics[width=\textwidth]{celebA_reconstr_500_orig_lasso-dct_lasso-wavelet_dcgan-gen.pdf}
            \caption{500 measurements}
        \end{subfigure}
        \begin{subfigure}[]{0.8\textwidth}
            \includegraphics[width=\textwidth]{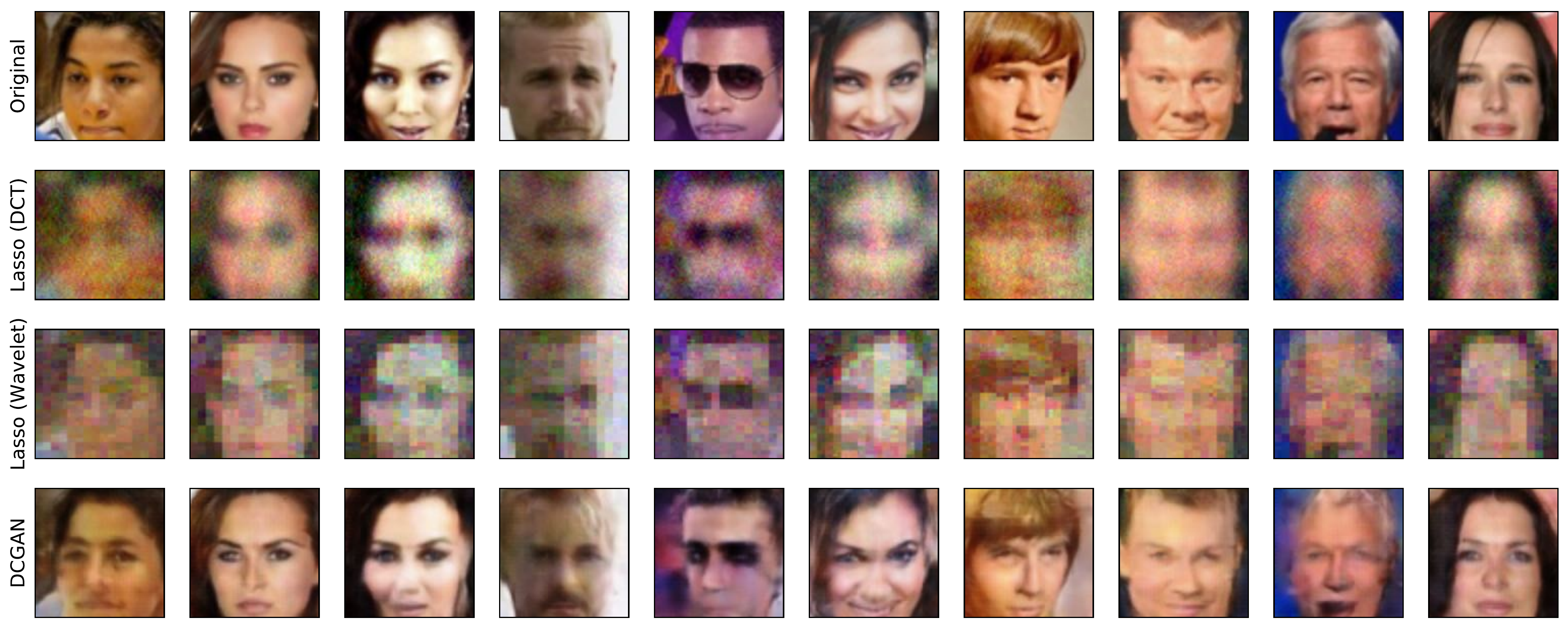}
            \caption{1000 measurements}
        \end{subfigure}
        \begin{subfigure}[]{0.8\textwidth}
            \includegraphics[width=\textwidth]{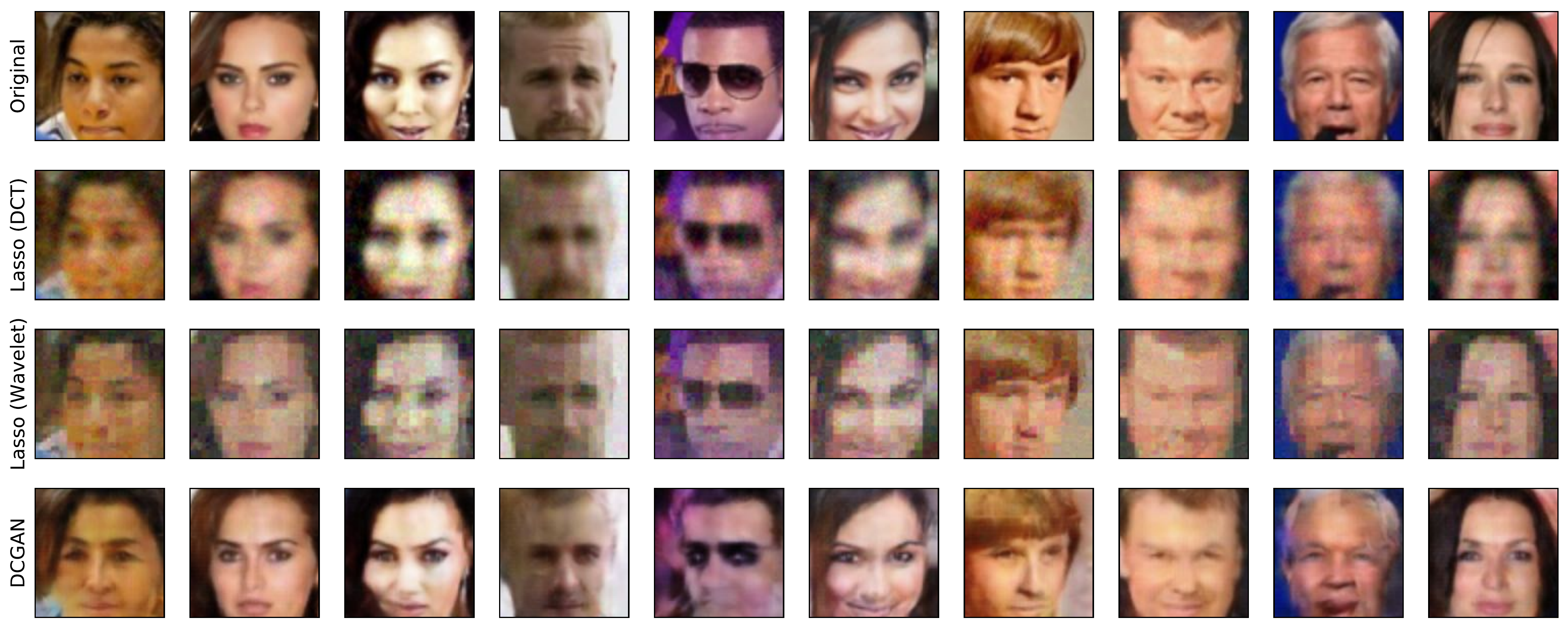}
            \caption{2500 measurements}
        \end{subfigure}
        \caption{Reconstruction on celebA. In each image, top row is ground truth, subsequent two rows show reconstructions by Lasso (DCT) and Lasso (Wavelet) respectively. The bottom row is the reconstruction by our algorithm.}
		\label{fig:more-celebA-reconstr2}
    \end{center}
\end{figure*}

\begin{figure*}
    \begin{center}
        \begin{subfigure}[]{0.8\textwidth}
            \includegraphics[width=\textwidth]{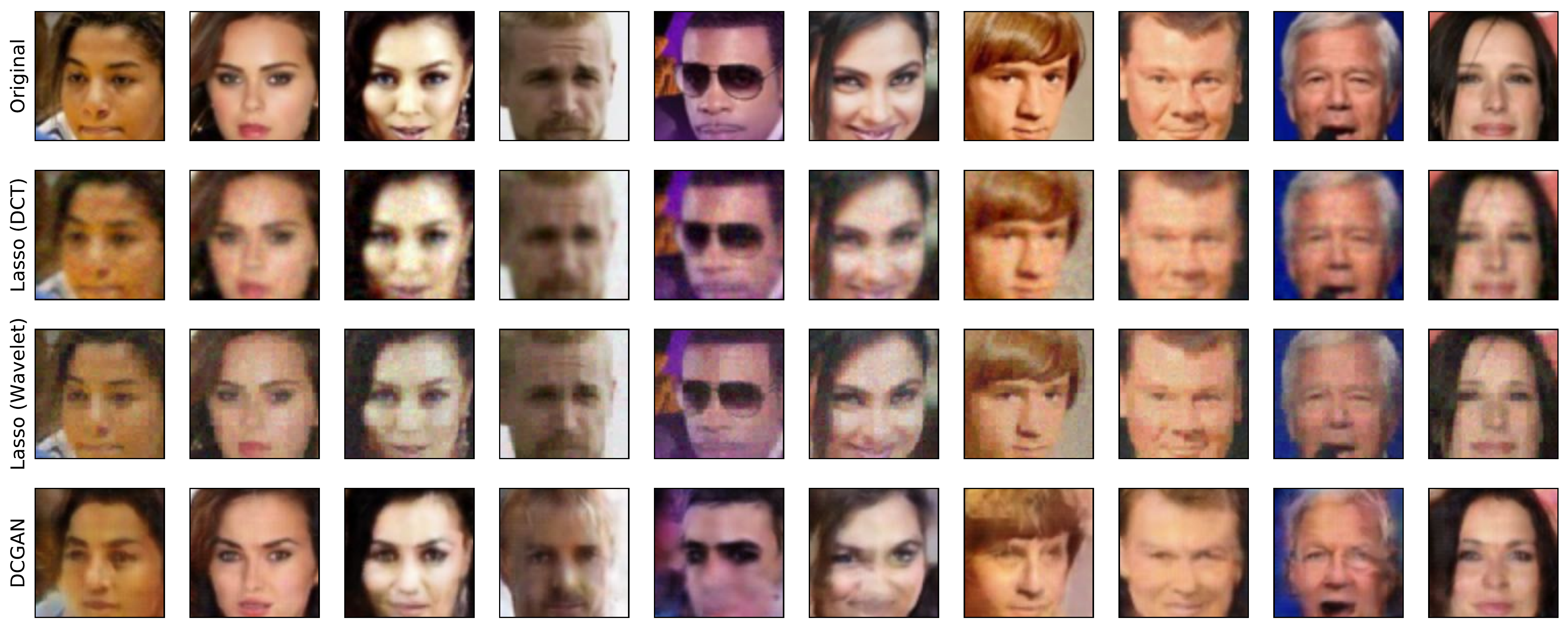}
            \caption{5000 measurements}
        \end{subfigure}
        \begin{subfigure}[]{0.8\textwidth}
            \includegraphics[width=\textwidth]{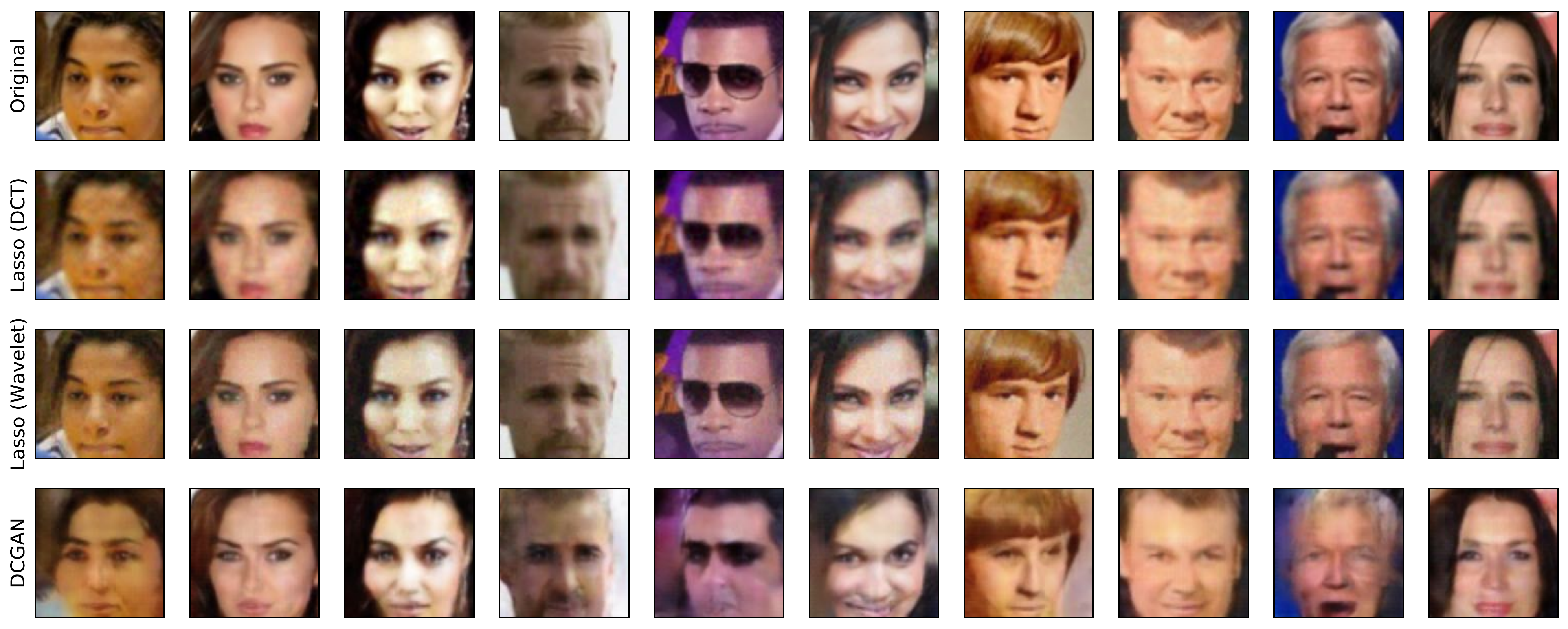}
            \caption{7500 measurements}
        \end{subfigure}
        \begin{subfigure}[]{0.8\textwidth}
            \includegraphics[width=\textwidth]{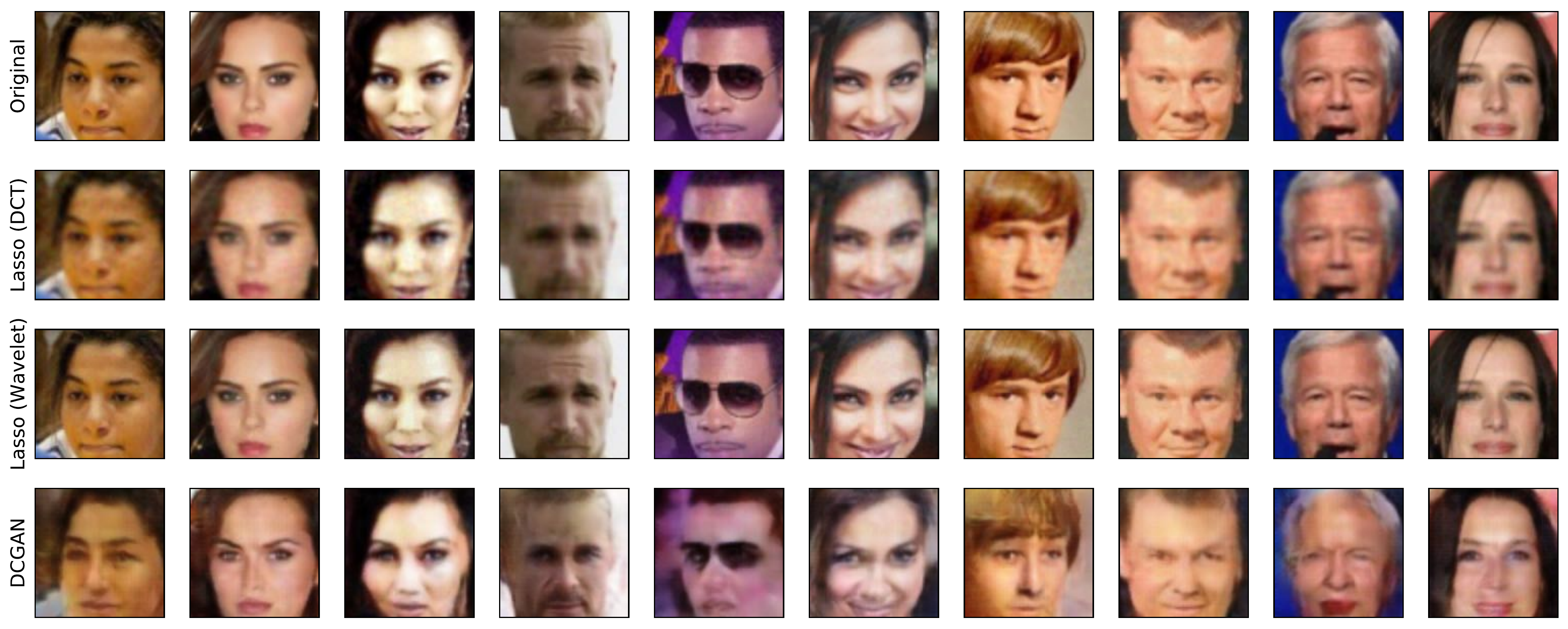}
            \caption{10000 measurements}
        \end{subfigure}
        \caption{Reconstruction on celebA. In each image, top row is ground truth, subsequent two rows show reconstructions by Lasso (DCT) and Lasso (Wavelet) respectively. The bottom row is the reconstruction by our algorithm.}
		\label{fig:more-celebA-reconstr3}
    \end{center}
\end{figure*}

\end{document}